\documentclass{article}

\usepackage[nonatbib, final]{neurips_2020}

\usepackage[utf8]{inputenc} 
\usepackage[T1]{fontenc}    
\usepackage{hyperref}       
\usepackage{url}            
\usepackage{booktabs}       
\usepackage{amsfonts}       
\usepackage{nicefrac}       
\usepackage{microtype}      
\usepackage{amsmath}

\usepackage{graphicx}
\usepackage{savesym}
\usepackage[linesnumbered, ruled,vlined]{algorithm2e}
\savesymbol{listofalgorithms}
\usepackage{algorithm,algorithmic}
\usepackage{filecontents}
\usepackage{epsdice}
\usepackage{hyperref}
\usepackage{bm}
\usepackage{wrapfig}

\usepackage{todonotes}

\newcommand{\pluseq}{\mathrel{+}=}

\newcommand{\method}[0]{RR}
\newcommand{\accmethod}[0]{RR}
\newcommand{\methodlong}[0]{Ridge Rider}

\usepackage{url}
\usepackage{graphicx}
\usepackage[font=scriptsize]{caption}
\usepackage{amsmath}

\usepackage{amsfonts}
\usepackage{graphicx} 
\usepackage{natbib}

\usepackage{xcolor}
\usepackage{tabularx} 
\usepackage{amsthm}
\usepackage{amsfonts}
\usepackage{amsmath}
\usepackage{amssymb}
\usepackage{mathrsfs}
\usepackage{multirow}
\usepackage{bbm}
\usepackage{wrapfig}

\newtheorem{theorem}{Theorem}
\newtheorem{definition}{Definition}

\newtheorem{proposition}{Proposition}
\newtheorem{lemma}{Lemma}

\usepackage{enumitem}
\usepackage{wrapfig}

\newcommand{\beq}{\begin{equation}}

\newcommand{\eeq}{\end{equation}}
\newcommand{\beqs}{\begin{equation*}}
\newcommand{\eeqs}{\end{equation*}}

\newcommand{\R}{\mathbb{R}}

\newcommand{\norm}[1]{\lVert #1 \rVert}

\newcommand{\del}{\delta}
\renewcommand{\th}{\theta}
\newcommand{\la}{\lambda}

\usepackage[framemethod=TikZ]{mdframed}
\usepackage{amsthm}
\newcounter{theo}[section] 
\renewcommand{\thetheo}{\arabic{section}.\arabic{theo}}

\newcounter{lem}[section] 
\renewcommand{\thelem}{\arabic{section}.\arabic{lem}}

\newcounter{prf}[section]
\renewcommand{\theprf}{\arabic{section}.\arabic{prf}}

\newcounter{cor}[section]
\renewcommand{\thecor}{\arabic{section}.\arabic{cor}}

\newcounter{prop}[section] 
\renewcommand{\thelem}{\arabic{section}.\arabic{prop}}

\renewcommand{\algorithmiccomment}[1]{\bgroup\hfill//~#1\egroup}

\usepackage{comment}

\usepackage{caption}
\newcommand\blfootnote[1]{%
  \begingroup
  \renewcommand\thefootnote{}\footnote{#1}%
  \addtocounter{footnote}{-1}%
  \endgroup
}

\title{Ridge Rider: Finding Diverse Solutions by Following Eigenvectors of the Hessian}

\author{%
  Jack Parker-Holder$^{\ast}$ \\
   University of Oxford 
   \And
   Luke Metz \\
   Google Research, Brain Team 
   \And
   Cinjon Resnick \\
   NYU 
   \And
   Hengyuan Hu \\
   FAIR
   \And
   Adam Lerer \\
   FAIR
   \And
   Alistair Letcher 
   \And
   Alex Peysakhovich \\
   FAIR
   \And
   Aldo Pacchiano \\
   BAIR
   \And
   Jakob Foerster$^{\ast}$\\
   FAIR
}

\begin{document}

\maketitle

\blfootnote{$^\ast$Equal contribution. Correspondence to \href{mailto:jackph@robots.ox.ac.uk}{jackph@robots.ox.ac.uk}, jnf@fb.com}

\begin{abstract}

Over the last decade, a single algorithm has changed many facets of our lives - Stochastic Gradient Descent (SGD). In the era of ever decreasing loss functions, SGD and its various offspring have become the go-to optimization tool in machine learning and are a key component of the success of deep neural networks (DNNs). 
While SGD is guaranteed to converge to a local optimum (under loose assumptions), in some cases it may matter \emph{which} local optimum is found, and this is often context-dependent. Examples frequently arise in machine learning, from shape-versus-texture-features to ensemble methods and zero-shot coordination. 
In these settings, there are desired solutions which SGD on `standard' loss functions will not find, since it instead converges to the `easy' solutions. 
In this paper, we present a different approach. Rather than following the gradient, which corresponds to a locally greedy direction, we instead follow the eigenvectors of the Hessian, which we call ``ridges''. 
By iteratively following and branching amongst the ridges, we effectively span the loss surface to find qualitatively different solutions.
We show both theoretically and experimentally that our method, called \emph{\methodlong{}} (\accmethod), offers a promising direction for a variety of challenging problems. 




\end{abstract}

\section{Introduction}
\label{sec:intro}
Deep Neural Networks (DNNs) are extremely popular in many applications of machine learning ranging from vision \citep{resnet, lenet} to reinforcement learning \citep{silver}. Optimizing them is a non-convex problem and so the use of gradient methods (e.g. stochastic gradient descent, SGD) leads to finding local minima. While recent evidence suggests \citep{choromanska2015loss} that in supervised problems these local minima obtain loss values close to the global minimum of the loss, there are a number of problem settings where optima with the same value can have very different properties. For example, in Reinforcement Learning (RL), two very different policies might obtain the same reward on a given task, but one of them might be more robust to perturbations. Similarility, it is known that in supervised settings some minima generalize far better than others \citep{DBLP:journals/corr/KeskarMNST16, flatminima}. Thus, being able to find a specific \emph{type} or class of minimum is an important problem.

At this point it is natural to ask what the benefit of finding diverse solutions is? Why not optimize the property we care about directly? The answer is \emph{Goodhart's law}: ``When a measure becomes a target, it ceases to be a good measure.'' \citep{strathern1997improving}. Generalization and zero-shot coordination are two examples of these type of objectives, whose very definition prohibits direct optimization.

To provide a specific example, in computer vision it has been shown that solutions which use shape features are known to generalize much better than those relying on textures \citep{DBLP:journals/corr/abs-1811-12231}. However, they are also more difficult to find \citep{Geirhos2020a}. In reinforcement learning (RL), recent work focuses on constructing agents that can coordinate with humans in the cooperative card game Hanabi \citep{DBLP:journals/corr/abs-1902-00506}. Agents trained with self-play find easy to learn, but highly arbitrary, strategies which are impossible to play with for a novel partner (including human). To avoid these undesirable minima previous methods need access to the symmetries of the problem to make them inaccessible during training. The resulting agents can then coordinate with novel partners, including humans \citep{hu2020other}. Importantly, in both of these two cases, standard SGD-based methods do not find these `good' minima easily and problem-specific hand tuning is required by designers to prevent SGD from converging to `bad' minima. 

Our primary contribution is to take a step towards addressing such issues in a general way that is applicable across modalities. One might imagine a plausible approach to finding different minima of the loss landscape would be to initialize gradient descent near a saddle point of the loss in multiple replicates, and hope that it descends the loss surface in different directions of negative curvature. Unfortunately, from any position near a saddle point, gradient descent will always curve towards the direction of most negative curvature (see Appendix \ref{app:gd_dynamics}), and there may be many symmetric directions of high curvature.

\begin{wrapfigure}{r}{0.43\textwidth} 
\label{fig:hat_motivation}
\vspace{-3mm}
    \centering
    \includegraphics[width=0.42\textwidth]{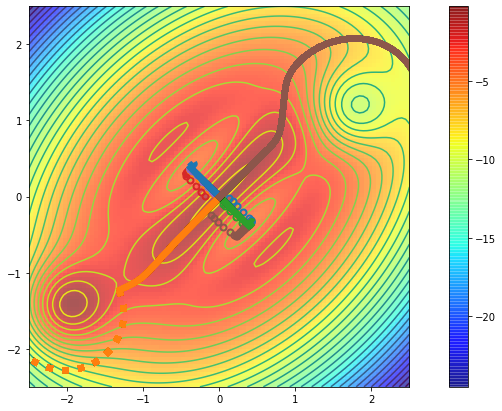}    
    \vspace{-3mm}
    \caption{\small Comparison of gradient descent (GD, hollow circles) and \method{} (\accmethod{}, solid circles) on a two-dimensional loss surface. GD starting near the origin only finds the two local minima whose basin have large gradient near the origin. \accmethod{} starts at the maximum and explores along four paths based on the two eigenvectors of $\mathcal{L}$. Two paths (blue and green) find the local minima while the other two explore the lower-curvature ridge and find global minima. Following the eigenvectors leads \accmethod{} around a local minimum (brown), while causing it to halt at a local maximum (orange) where either a second ride (dotted) or GD may find a minimum.
    }
    \vspace{-6mm}
    \label{fig:rr_path}
\end{wrapfigure}


Instead, we start at a saddle point and \textit{force} different replicates to follow distinct, orthogonal directions of negative curvature by iteratively following each of the eigenvectors of the Hessian until we can no longer reduce the loss, at which point we repeat the branching process. Repeating this process hopefully leads the replicates to minima corresponding to distinct convex subspaces of the parameter space, essentially converting an optimization problem into search \cite{branchandbound}. We refer to our procedure as \methodlong{} (\accmethod{}). \accmethod{} is less `greedy' than standard SGD methods, with Empirical Risk Minimization (ERM, \citep{erm_vapnik}), and thus can be used in a variety of situations to find diverse minima. This greediness in SGD stems from it following the path with highest expected local reduction in the loss. Consequently, some minima, which might actually represent the solutions we seek, are never found.


In the next section we introduce notation and formalize the problem motivation. In Section \ref{sec:method} we introduce \accmethod{}, both as an exact method, and as an approximate scalable algorithm. In both cases we underpin our approach with theoretical guarantees. Finally, we present extensions of \accmethod{} which are able to solve a variety of challenging machine learning problems, such as finding diverse solutions reinforcement learning, learning successful strategies for zero-shot coordination and generalizing to out of distribution data in supervised learning. We test \accmethod{} in each of these settings in Section \ref{sec:experiments}. Our results suggest a conceptual connection between these previously unrelated problems.
\vspace{-3mm}


\section{Background}
\label{sec:background}
Throughout this paper we assume a smooth, \emph{i.e.},  infinitely differentiable, loss function $\mathcal{L}_{\theta} =\mathbb{E}_{\mathbf{x}_i} L_{\theta}(\mathbf{x}_i)$, where $\theta \in  \mathbb{R}^n = \Theta$ are typically the weights of a DNN.
In the supervised setting, $\mathbf{x}_i = \{x_i, y_i\}$ is an input and label for a training data point, while $L_{\theta}(\mathbf{x}_i)$ could be the cross-entropy between the true label $y_i$ and the prediction $f_\theta(x_i)$.

We use $\nabla_\theta \mathcal{L}$ for the gradient and $\mathcal{H}$ for the Hessian, ie. $\mathcal{H} = \nabla^2_\theta \mathcal{L}$.
The eigenvalues (EVals) and eigenvectors (EVecs) of $\mathcal{H}$ are $\lambda_i$ and $e_i$ respectively. The computation of the full Hessian is prohibitively expensive for all but the smallest models, so we assume that an automatic differentiation library is available from which vector-Hessian products, $\mathcal{H} v$, can be computed efficiently~\citep{pearlmutter1994fast}: $\mathcal{H} v = \nabla_\theta \big( ( \nabla_\theta  \mathcal{L}) v \big).$

\paragraph{Symmetry and Equivalence}
Real world problems commonly contain symmetries and invariances. For example, in coordination games, the payout is unchanged if all players jointly update their strategy to another equilibrium with the same payout. We formalize this as follows: A symmetry, $\phi$, of the loss function is a bijection on the parameter space such that $\mathcal{L}_\theta = \mathcal{L}_{\phi(\theta)},  \text{ for all } \theta \in \Theta$.

If there are $N$ non-overlapping sets of $m$ parameters each, $k_i = \{\theta_{k_i^1}, .. \theta_{k_i^m}\}, i \in \{1, ... N\}$ and the loss function is invariant under all permutations of these $N$ sets, we call these sets `N-equivalent'. In a card-game which is invariant to color/suit, this corresponds to permuting both the color-dependent part of the input layer and the output layer of the DNN simultaneously for all players.



\paragraph{Zero-Shot Coordination} The goal in Zero-Shot Coordination is to coordinate with a stranger on a fully cooperative task during a single try. Both the problem setting and the task are common knowledge. A zero-shot coordination scheme (or learning rule) should find a policy that obtains high average reward when paired with the distribution of policies obtained independently but under the same decision scheme. All parties can agree on a scheme beforehand, but then have to obtain the policy independently when exposed to the environment. A key challenge is that problems often have symmetries, which can generate a large set of equivalent but mutually incompatible policies.

We assume we have a fully observed MDP with states $s_t \in \mathcal{S}$ and agents $\pi_{i \in [1, N]}$, each of whom chooses actions $a_t^i \in \mathcal{A}$ at each step. The game is cooperative with agents sharing the reward $r_t$ conditioned on the joint action and state, and the goal is to maximize expected discounted return $J = \mathbb{E}_\tau \sum_t \gamma^{t}r_t$, where $\gamma$ is the discount factor and $\tau$ is the trajectory.

\paragraph{Out of Distribution Generalization (OOD)}
\label{subsection:oodg}
We assume a multi-environment setting where our goal is to find parameters $\theta$ that perform well on all $N$ environments given a set $m < N$ of training environments $\Xi = \{\xi_0, \xi_1, \ldots, \xi_m\}$. The loss function $\ell$ and the domain and ranges are fixed across environments. Each environment $\xi$ has an associated dataset $D_\xi$ and data distribution $\mathcal{D}_\xi$. 
Together with our global risk function, $l$, these induce a per environment loss function,
\begin{align*}
    \mathcal{L}_{\xi}(\theta) = \mathbb{E}_{\mathbf{x}_i \sim \mathcal{D}_\xi} \ell(\mathbf{x}_i).
\end{align*}
Empirical Risk Minimization (ERM) ignores the environment and minimizes the average loss across all training examples, 
\begin{align*}
    \mathcal{L}_{ERM}(\theta) = \mathbb{E}_{\mathbf{x}_i \sim \cup_{\xi \in \Xi} \mathcal{D}_\xi} \ell_{\xi}(\mathbf{x}_i).
\end{align*}

ERM can fail when the test distribution differs from the train distribution.
\section{Method}
\label{sec:method}

We describe the intuition behind \accmethod{} and present both exact and approximate (scalable) algorithms. We also introduce extensions to zero-shot coordination in multi-agent settings and out of distribution generalization.

The goal of \accmethod{} is as follows. Given a smooth loss function, $\mathcal{L}(\theta)$, discover qualitatively different local minima, $\theta^*$, while grouping together those that are equivalent (according to the definition provided in Section~\ref{sec:background}) or, alternatively, only exploring one of each of these equivalent policies. 

While symmetries are  problem specific, in general \emph{equivalent parameter sets} (see Section~\ref{sec:background}) lead to repeated EVals. If a given loss function in $\mathbb{R}^n$ has $N$-equivalent parameter sets  ($N > 2$) of size $m$ and a given $\theta$ is invariant under all the associated permutations, then the Hessian at $\theta$ has at most $n-m(N-2)$ distinct eigenvalues  (proof in Appendix~\ref{subsection::repeated_eigenvalues}). Thus, rather than having to explore up to $n$ orthogonal directions, in some settings it is sufficient to explore one member of each of the groups of  distinct EVals to obtain different  non-symmetry-equivalent solutions. Further, the EVecs can be ordered by the corresponding EVal, providing a numbering scheme for the classes of solution.

To ensure that there is at least one negative EVal and that all negative curvature directions are locally loss-reducing, we start at or near a strict saddle point $\theta^{\text{MIS}}$.  
For example, in supervised settings, we can accomplish this by initializing the network near zero. 
Since this starting point combines small gradients and invariance, we refer to it as the Maximally Invariant Saddle (MIS). Formally, 
\begin{align*}
    \theta^{\text{MIS}} = \arg\min_{\theta} |\nabla_{\theta} J(\theta)|, \text{ s.t. } \phi(\theta ) = \theta, \text{  } \forall \phi,
\vspace{-3mm}
\end{align*}
for all symmetry maps $\phi$ as defined in Section~\ref{sec:background}.

In tabular RL problems the MIS can be obtained by optimizing the following objective (Proof in Appendix~\ref{subsection::MIS}): 
\begin{align*}
\vspace{-5mm}
    \theta^{\text{MIS}}= \arg\min_{\theta}  |\nabla_{\theta} J(\theta)|  - \lambda H(\pi_\theta(\mathbf{a})), \lambda > 0
\end{align*}
\vspace{-5mm}


From $\theta^{\text{MIS}}$, \accmethod{} proceeds as follows: We branch to create replicas, which are each updated in the direction of a different EVec of the Hessian (which we refer to as `ridges'). While this is locally loss reducing, a single step step typically does not solve a problem. Therefore, at all future timesteps, rather than choosing a new EVec, each replicate follows the updated version of its original ridge until a break-condition is met, at which point the branching process is repeated.  For any ridge this procedure is repeated until a locally convex region without any negative curvature is found, at which point gradient descent could be performed to find a local optimum or saddle. In the latter case, we can in principle apply \accmethod{} again starting at the saddle (although we did not find this to be necessary in our setting). We note that \accmethod{} effectively turns a continuous optimization problem, over $\theta$, into a discrete search process, \emph{i.e.}, which ridge to explore in what order.

During \accmethod{} we keep track of a \emph{fingerprint}, $\Psi$, containing the indices of the EVecs chosen at each of the preceding branching points. $\Psi$ uniquely describes $\theta^\Psi$ up to repeated EVals. In Algorithm \ref{algorithm:diversityrr} we show pseudo code for RR, the functions it takes as input are described in the next paragraphs.

$\mathrm{UpdateRidge}$ computes the updated version of $e_i$ at the new parameters. The EVec $e_i$ is a continuous function of $\mathcal{H}$ (\cite{kato1995perturbation}, pg 110-111).\footnote{We rely on the fact that the EVecs are continuous functions of $\theta$; this follows from the fact that $\mathcal{L}(\theta)$ is continuous in $\theta$, and that $\mathcal{H}(\mathcal{L})$ is a continuous function of $\mathcal{L}$.} This allows us to `follow' $e_i(\theta^\Psi_t)$ (our `ridge') for a number of steps 
even though it is changing. In the exact version of RR, we recompute the spectrum of the Hessian after every parameter update, find the EVec with greatest overlap, and then step along this updated direction. While this updated EVec might no longer correspond to the $i$-th EVal, we maintain the subscript $i$ to index the ridge. 
The dependency of $e_i(\theta)$ and $\lambda_i(\theta)$ on $\theta$ is entirely implicit: $\mathcal{H}(\theta) e_i(\theta) = \lambda_i(\theta)e_i(\theta), \text{  } |e_i| = 1$.

 $\mathrm{EndRide}$ is a heuristic that determines how long we follow a given ridge for. For example, this can consider whether the curvature is still negative, the loss is decreasing and other factors. 

$\mathrm{GetRidges}$ determines which ridges we explore from a given branching point and in what order. 
Note that from a saddle, one can explore in \textit{opposite} directions along any negative EVec. 
Optionally, we select the $N$ most negative EVals.  

$\mathrm{ChooseFromArchive}$ provides the search-order over all possible paths. For example, we can use breadth first search (BFS), depth first search (DFS) or random search. In BFS, the archive is a FIFO queue, while in DFS it is a LIFO queue.  
Other orderings and heuristics can be used to make the search more efficient, such as ranking the archive by the current loss achieved. 

\vspace{-2mm}
\begin{algorithm}[H]
\SetAlgoLined
\caption{\methodlong{}}
\begin{algorithmic}[1]

\STATE \textbf{Input:}  $\theta^{\text{MIS}}$, $\alpha$, $\mathrm{ChooseFromArchive}, \mathrm{UpdateRidge}, \mathrm{EndRide}$, $\mathrm{GetRidges}$ 
\STATE $A=[ \{\theta^{\Psi = [i]}, e_i, \lambda_i \} \; \text{ for } i, e_i, \lambda_i \in \mathrm{GetRidges}(\theta^{\text{MIS}})]$ \COMMENT{Initialize Archive of Solutions}
\WHILE{$|A| > 0$} 
  \STATE $\{e_i, \theta^\Psi_0, \lambda_i \}, A  = \mathrm{ChooseFromArchive}(A)$ \COMMENT{Select a ridge from the archive}
  \WHILE{$\mathrm{True}$} 
    \STATE $\theta^\Psi_{t} = \theta^\Psi_{t-1} - \alpha e_i$ \COMMENT{Step along the Ridge with learning rate $\alpha$}
    \STATE $e_i, \lambda_i = \mathrm{UpdateRidge}(\theta^\Psi_t, e_i, \lambda_i)$ \COMMENT{Get updated Ridge}
    \IF{$\mathrm{EndRide}(\theta^\Psi_t, e_i, \lambda_i)$ } 
      \STATE break \COMMENT{Check Break Condition}
    \ENDIF
  \ENDWHILE 
  \STATE $A=A \cup \left[ \{\theta^{\Psi.\mathrm{append}(i)}, e_i, \lambda_i \} \; \text{ for } i, e_i, \lambda_i \in \mathrm{GetRidges}(\theta^\Psi ) \right]$ \COMMENT{Add new Ridges}
\ENDWHILE

\end{algorithmic}
\label{algorithm:diversityrr}
\end{algorithm}
\vspace{-4mm}

It can be shown that, under mild assumptions, \accmethod{} maintains a descent direction: At $\theta$, its EVals are $\lambda_1(\theta)  \geq \lambda_2(\theta) \cdots \geq \lambda_d(\theta)$ with EVecs $e_1(\theta), \cdots, e_d(\theta)$. We denote the eigengaps as $\Delta_{i-1} := \lambda_{i-1}(\theta) - \lambda_{i}(\theta) $ and $\Delta_i := \lambda_{i}(\theta) - \lambda_{i+1}(\theta) = \Delta_{i} $, with the convention $\lambda_0  = \infty$.

\begin{theorem}
Let $L: \Theta \rightarrow \mathbb{R}$ have $\beta-$smooth Hessian (i.e. $\| \mathcal{H}(\theta) \|_{op} \leq \beta$ for all $\theta$), let $\alpha$ be the step size. If $\theta$ satisfies:
$\langle \nabla L(\theta), e_i(\theta) \rangle \geq  \| \nabla L(\theta) \| \gamma $ for some $\gamma \in (0,1)$, and $\alpha \leq \frac{\min(\Delta_i, \Delta_{i-1}) \gamma^2 }{16\beta} $ then after two steps of RR:
\begin{equation*}
L(\theta'') \leq     L(\theta)  - \gamma \alpha\| \nabla L(\theta) \| 
\end{equation*}
Where $\theta' = \theta - \alpha e_i(\theta)$ and $\theta'' = \theta' - \alpha e_i(\theta')$.
\end{theorem}
In words, as long as the correlation between the gradient and the eigenvector \accmethod{} follows remains large, the slow change in eigenvector curvature will guarantee \accmethod{} remains on a descent direction.
Further, starting from any saddle, after $T$-steps of following $e_i(\theta_t)$, the gradient is $\nabla \mathcal{L}(\theta_T) = \alpha \sum_{t} \lambda_i(\theta_t) e_i(\theta_t) + \mathcal{O}(\alpha^2)$. Therefore, $ \langle \nabla \mathcal{L}(\theta), e_i(\theta_T) \rangle = \alpha \sum_{t} \lambda_i(\theta_t) \langle e_i(\theta_t), e_i(\theta_T) \rangle  + \mathcal{O}(\alpha^2)$
Thus, assuming $\alpha$ is small, a sufficient condition for reducing the loss at every step is that $\lambda_i(\theta_t) <0,  \forall t$ and the $e_i(\theta_t)$ have positive overlap, $\langle e_i(\theta_t),e_i(\theta_{t'})  \rangle > 0,  \forall t, t'$. Proofs are in Appendix~\ref{subsection::staying_on_ridge}.

\textbf{Approximate  \method{}:}
\label{subsection:scalablerr}
In exact \accmethod{} above, we assumed that we can compute the Hessian and also obtain all EVecs and EVals. 
To scale \accmethod{} to large DNNs, we make two modifications. First, in $\mathrm{GetRidges}$ we use the power method (or Lanczos method \citep{lanczos}) to obtain approximate versions of the N most negative $\lambda_i$ and corresponding $e_i$. Second, in $\mathrm{UpdateRidge}$ we use gradient descent after each parameter update $\theta^\Psi \rightarrow \theta^\Psi - \alpha e_i$ to yield a new $e_i$, $\lambda_i$ pair that minimizes the following loss:
\begin{align*}
     L(e_i, \lambda_i ; \theta)&= |(1/\lambda_i) \mathcal{H}(\theta)e_i / |e_i| - e_i / |e_i| |^2 
\end{align*}
We warm-start with the 1st-order approximation to $\lambda(\theta)$, where $\theta', \lambda', e'_i$ are the previous values:
\begin{align*}
    \lambda_i(\theta) &\approx  \lambda'_i + e'_i \Delta \mathcal{H} e'_i = \lambda'_i +  e'_i (\mathcal{H}(\theta)- \mathcal{H}(\theta') )  e'_i 
\end{align*}
Since these terms only rely on Hessian-Vector-products, they can be calculated efficiently for large scale DNNs in any modern auto-diff library, \emph{e.g.} Pytorch \citep{pytorch}, Tensorflow~\citep{abadi2016tensorflow} or Jax \citep{jax2018github}. See Algorithm~\ref{algorithm:approximaterr} in the Appendix (Sec. \ref{sec:exp_details}) for pseudocode. 

We say $e_0 = e_i(\theta)$ and $e_t$ is the $t-$th EVal in the algorithm's execution. This algorithm has the following convergence guarantees, \emph{ie.}:

\begin{theorem}
If $L$ is $\beta-$smooth, $\alpha_e = \min(1/4, \Delta_{i}, \Delta_{i-1})$, and $\| \theta - \theta'\| \leq \frac{\min(1/4, \Delta_{i}, \Delta_{i-1})}{\beta} $ then $|\langle e_t, e_i(\theta') \rangle| \geq 1- \left(1-\frac{\min(1/4, \Delta_{i}, \Delta_{i-1})}{4} \right)^{t} $
\end{theorem}
This result characterizes an exponentially fast convergence for the approximate  \method{} optimizer. If the eigenvectors are well separated, UpdateRidge will converge faster. The proof is in Appendix~\ref{section::convergence_eigenvalue_eigenvector}.

\textbf{\method{} for Zero-Shot Coordination in Multi-Agent Settings:}
\label{subsection:zeroshotrr}
\accmethod{} provides a natural decision scheme for this setting -- decide in advance that each agent will explore the top $F$ fingerprints. For each fingerprint, $\Psi$, run $N$ independent replicates $\pi$ of the \accmethod{} procedure and compute the average cross-play score among the $\pi$ for each $\Psi$. At test time, deploy a $\pi$ corresponding to a fingerprint with the highest score. Cross-play between two policies, $\pi^a$ and $\pi^b$, is the expected return, $J(\pi^a_1,\pi^b_2) $, when agent one plays their policy of $\pi^a$ with the policy for agent two from $\pi^b $.

This solution scheme relies on the fact that the ordering of unique EVals is consistent across different runs. Therefore, fingerprints corresponding to polices upon which agents can reliably coordinate will produce mutually compatible polices across different runs. Fingerprints corresponding to arbitrary symmetry breaking will be affected by inconsistent EVal ordering since EVals among equivalent directions are equal. 

Consequently, there are two key insights that makes this process succeed without having to know the symmetries. The first is that the MIS initial policy is invariant with respect to the symmetries of the task, and the second is that equivalent parameter sets lead to repeated EVals.

\textbf{Extending  \method{} for Out of Distribution Generalization:} 
\label{subsec:OOD} Consider the following coordination game. Two players are each given access to a non-overlapping set of training environments, with the goal of learning consistent features. 
While both players can agree beforehand on a training scheme and a network initialization, they cannot communicate after they have been given their respective datasets. This coordination problem resembles OOD generalization in supervised learning.

\accmethod{} could be adapted to this task by finding solutions which are reproducible across all datasets. One necessary condition is that the EVal and EVec being followed is consistent across training environments $\xi \in \Xi$ to which each player has access:
\begin{align*}
    \mathcal{H}_\xi e = \lambda e, \text{ }  \forall \xi \in \Xi,  \text{ }  \lambda < 0
\end{align*}
 where   $\mathcal{H}_\xi$  is the Hessian of the loss evaluated on environment  $\xi$, \emph{i.e.}, $ \nabla^2_\theta \mathcal{L}_{\xi}$.
Unfortunately, such $e, \lambda$ do not typically exist since there are no consistent features present in the raw input. 
To address this, we extend \accmethod{} by splitting the parameter space into $\Theta_f$ and $\Theta_r$. The former embeds inputs as \textbf{f}eatures for the latter, creating an abstract feature space representation in which we can run RR.

For simplicity, we consider only two training environments with, respectively, Hessians $\mathcal{H}^1_r$ and $\mathcal{H}^2_r$. The $R$ indicates that we are only computing the Hessian in the subspace $\Theta_r$.
Since we are not aware of an efficient and differentiable method for finding common EVecs, we parameterize a differentiable loss function to optimize for  an approximate common EVec, $e_r$, of the Hessians of all training environments in $\Theta_r$. This loss function forces high correlation between $H_r e_r$ and $e_r$ for both environments, encourages negative curvature, prevents the loss from increasing, and penalizes differences in the EVals between the two training environments:
\begin{align*}
    \mathcal{L}_{1}(\theta_{f}, e_r | \theta_r) =\sum_{i \in 1,2} \big( -\beta_1 \mathrm{C}(\mathcal{H}^i_r e_r, e_r) - \beta_2 e_r \mathcal{H}^i_r e_r + \beta_3 {L}_{\xi_i}(\theta_{f}| \theta_r)\big) + \frac{|e_r ( \mathcal{H}^1_r - \mathcal{H}^2_r) e_r   |}{|e_r|^2}.
\end{align*}

Here $\mathrm{C}$ is the correlation (a normalized inner product) and the ``$( \cdot |\theta_r)$'' notation indicates that $\theta_r$ is not updated when minimizing this loss, which can be done via stop\_gradient. 
All $\beta_i$ are hyperparameters. The minimum of $\mathcal{L}_{1}(\theta_{f}, e_r | \theta_r)$ is a consistent, negative EVal/EVec pair with low loss.

For robustness, we in parallel train $\theta_f$ to make the Hessian in $\theta_r$ consistent across the training environments in other directions. We do this by sampling random unit vectors, $u_r$ from $\Theta_r$ and comparing the inner products taken with the Hessians from each environment. The loss, as follows, has a global optimum at $0$ when $\mathcal{H}^1 = \mathcal{H}^2$:
\begin{align*}
    \mathcal{L}_{2}(\theta_{f} | \theta_r) =\mathbb{E}_{u_r\sim\Theta_r} \beta_4 \frac{|\mathrm{C}(\mathcal{H}^1_r u_r, u_r)^2 - \mathrm{C}(\mathcal{H}^2_r u_r, u_r)^2 |}{\mathrm{C}(\mathcal{H}^1_r u_r, u_r)^2 +\beta_5 \mathrm{C}(\mathcal{H}^1_r u_r, u_r)^2}+ \frac{|u_r ( \mathcal{H}^1_r - \mathcal{H}^2_r) u_r|}{|u_r  \mathcal{H}^1_r u_r| + |u_r \mathcal{H}^2_r  u_r|}.
\end{align*}
\accmethod{} now proceeds by starting with randomly initialized $e_r$, $\theta_f$, and $\theta_r$. Then we iteratively update $e_r$ and $\theta_f$ by running n-steps of SGD on $\mathcal{L}_{1} + \mathcal{L}_{2}$. We then run a step of \accmethod{} with the found approximate EVec $e_r$ and repeat. Pseudo-code is provided in the Appendix (Sec. \ref{sec:ood}).

\textbf{Goodhart's Law, Overfitting and Diverse Solutions}
\label{subsection:goodharts}
As mentioned in Section~\ref{sec:intro}, \method{} directly relates  to Goodhart's law, which states that any measure of progress fails to be useful the moment we start optimizing for it. So while it is entirely legitimate to use a validation set to estimate the generalization error for a DNN trained via SGD \emph{after the fact}, the moment we directly optimize this performance via SGD it seizes to be informative. 

In contrast, \method{} allows for a two step optimization process: We first produce a finite set of diverse solutions using only the training set and then use the validation data to chose the best one from these. Importantly, at this point we can use any generalization bound for finite hypothesis classes to bound our error ~\citep{mohri2018foundations}. For efficiency improvements we can also use the validation performance to locally guide the search process, which makes it unnecessary to actually compute all possible solutions.

Clearly, rather than using ~\accmethod{} we could try to produce a finite set of solutions by running SGD many times over. However, typically this would produce the same type of solution and thus fail to find those solutions that generalize to the validation set.

\section{Experiments}
\label{sec:experiments}

We evaluate \accmethod{} in the following settings: exploration in RL, zero-shot coordination, and supervised learning on both MNIST and the more challenging Colored MNIST problem \cite{arjovsky2019invariant}. In the following section we introduce each of the settings and present results in turn. Full details for each setting are given in the Appendix (Sec. \ref{sec:expdetails}).

\paragraph{\method{} for Diversity in Reinforcement Learning}
\label{subsection:diversityrr} To test whether we can find diverse solutions in RL, we use a toy binary tree environment with a tabular policy (see Fig. \ref{fig:tree_results}). The agent begins at $s_1$, selects actions $a\in \{\mathrm{left}, \mathrm{right}\}$, receiving reward $r \in \{-1, 10\}$ upon reaching a terminal node.   For the loss function, we compute the expectation of a policy as the sum of the rewards of each terminal node, weighted by the cumulative probability of reaching that node. The maximum reward is $10$. 

We first use the exact version of \method{} and find $\theta^\text{MIS}$, by maximizing entropy. For the $\mathrm{ChooseFromArchive}$ precedure, we use BFS. In this case we have access to exact gradients so can cheaply re-compute the Hessian and EVecs. As such, when we call $\mathrm{UpdateRidge}$ we adapt the learning rate $\alpha$ online to take the largest possible step while preserving the ridge, similar to Backtracking Line Search \cite{numericalopt}. We begin with a large $\alpha$, take a step, recompute the EVecs of the Hessian and the maximum overlap $\delta$. We then and sequentially halve $\alpha$ until we find a ridge satisfying the $\delta_\mathrm{break}$ criteria (or $\alpha$ gets too small). 
In addition, we use the following criteria for $\mathrm{EndRide}$: (1) If the dot product between $e_i$ and $e'_i$ is less than $ \delta_{\text{break}}$ (2) If the policy stops improving. 

We run \method{} with a maximum budget of $T=10^5$ iterations  similarity $\delta_\mathrm{break}=0.95$, and take only the top $N=6$ in $\mathrm{GetRidges}$. As baselines, we use gradient descent (GD) with random initialization, GD starting from the MIS, and random norm-one vectors starting from the MIS. All baselines are run for the same number of timesteps as used by \method{} for that tree. For each depth $d\in \{4,6,8,10\}$ we randomly generate $20$ trees and record the percentage of positive solutions found. 
\vspace{-2mm}
\begin{figure}[H]
        \begin{minipage}{0.99\textwidth}
        \vspace{-3mm}
        \centering\includegraphics[width=0.18\linewidth]{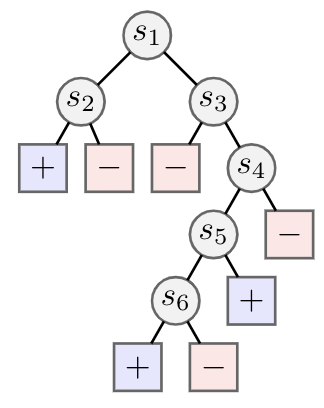}
        \centering\includegraphics[width=0.78\linewidth]{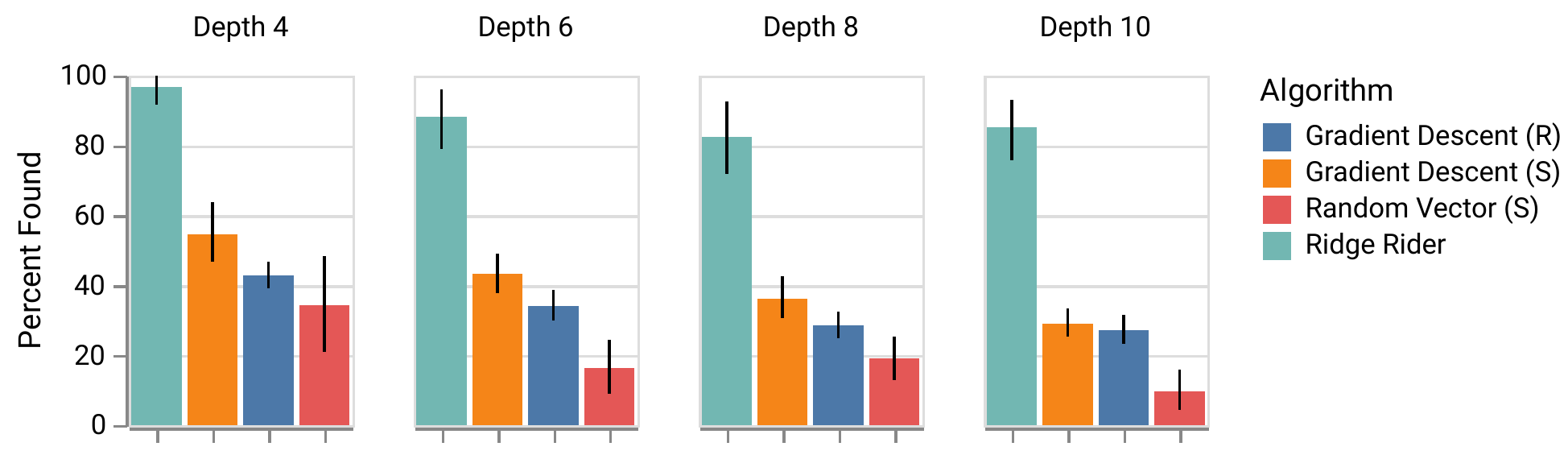}
        \end{minipage}
        \caption{\textbf{Left}: a tree with six decision nodes and seven terminal nodes, four of which produce negative rewards (red) and three of which produce positive rewards (blue). \textbf{Right}: The percentage of solutions found per algorithm, collated by tree depth. R and S represent starting from a random position or from a saddle, respectively. Trees at each depth are randomly generated $20$ times to produce error estimates shown.}
        \label{fig:tree_results}
\end{figure}
\vspace{-4mm}

\begin{wrapfigure}{r}{0.32\textwidth} 
\label{fig:ablation}
\vspace{-8mm}
    \centering
    \includegraphics[width=0.32\textwidth]{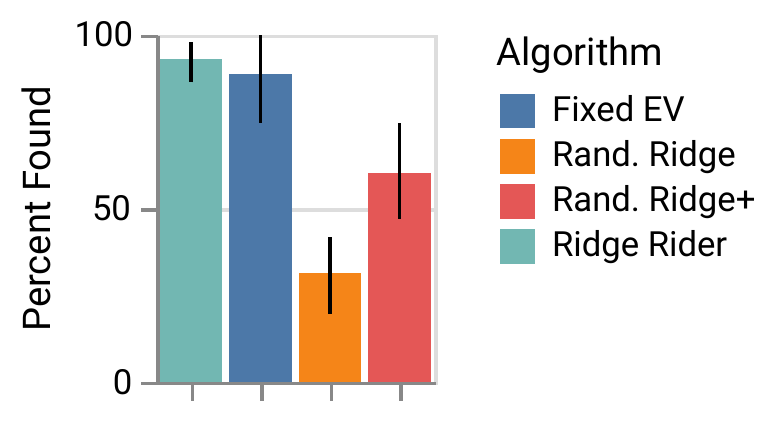}    
    \vspace{-7mm}
    \caption{\small Tree depth 12, ten seeds.}
    \vspace{-6mm}
    \label{fig:ablation_rl}
\end{wrapfigure}

On the right hand side of Fig~\ref{fig:tree_results}, we see that \method{} outperforms all three baselines. While \method{} often finds over $90\%$ of the solutions, GD finds at most $50\%$ in each setting. Importantly, following random EVecs performs poorly, indicating the importance of using the EVecs to explore the parameter space. To run this experiment, see the notebook at \url{https://bit.ly/2XvEmZy}.

Next we include two additional baselines: (1) following EVecs, but not updating them (\emph{Fixed-EV}). (2) following random unit vectors with \textit{positive ascent direction} (\emph{Rand-Ridge+}), and compare vs. \accmethod{}.
We ran these with a fixed budget, for a tree of depth 12. We used the same hyperparameters for \method{} and the ablations. As we see in Fig. \ref{fig:ablation_rl}, Fixed-EVs obtains competitive performance. This clearly illustrates the importance of following EVs rather than random directions. 

Finally, we open the door to using \accmethod{} in deep RL by computing Hessians using samples, leveraging more accurate higher order gradients produced by the DiCE objective \citep{dice}. Once again, \accmethod{} is able to find more diverse solutions than SGD (see Fig \ref{fig:dice_rl}).

\textbf{\method{} for Supervised Learning} We applied approximate \accmethod{} to MNIST with a 2-layer MLP containing 128 dimensions in the hidden layer. As we see on the left hand side of Fig~\ref{fig:mnist_plots}, we found that we can achieve respectable performance of approximately $98\%$ test and train accuracy. Interestingly, updating $e$ to follow the changing eigenvector is crucial. A simple ablation which sets LR$_{e}$ to 0 fails to train beyond $90\%$ on MNIST, even after a large hyper parameter sweep (see the right side of Fig~\ref{fig:mnist_plots}). We also tested other ablations. As in RL, we consider using random directions. Even when we force the random vectors to be ascent directions (Rand.Ridge +), the accuracy does not exceed 30\%. In fact, the outperformance from RL is more pronounced in MNIST, which is intuitive since random search is known to scale poorly to high dimensional problems. We expect this effect to be even more pronounced as Approximate \method{} is applied to harder and higher dimensional tasks in the future.

 \begin{figure*}[ht]
    \centering
    \includegraphics[width=0.3\textwidth]{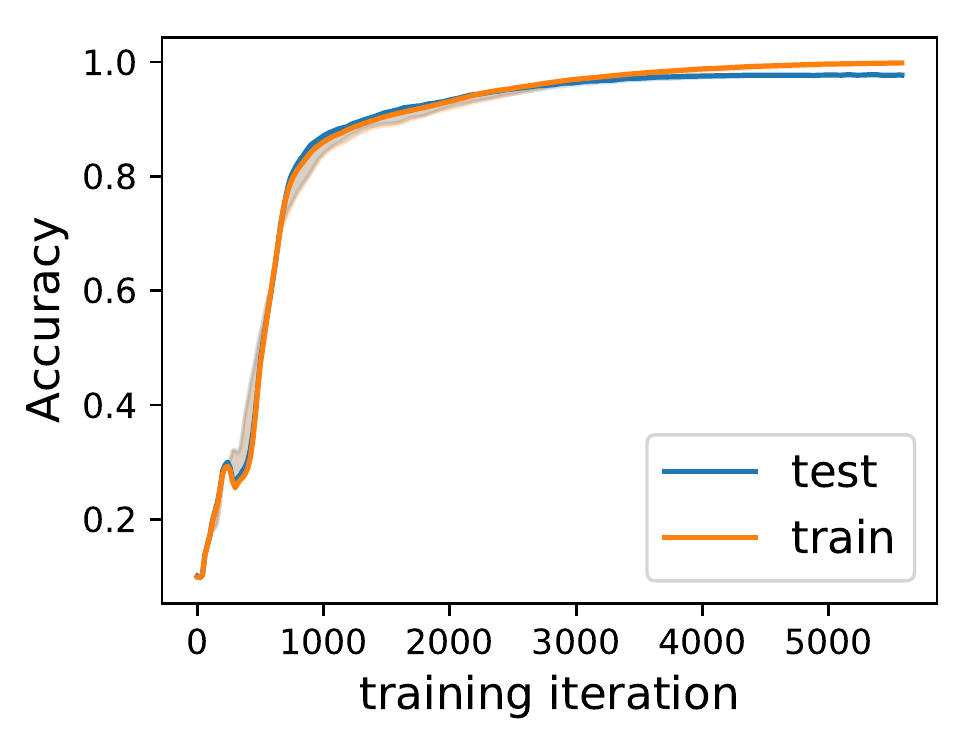}
    \includegraphics[width=0.3\textwidth]{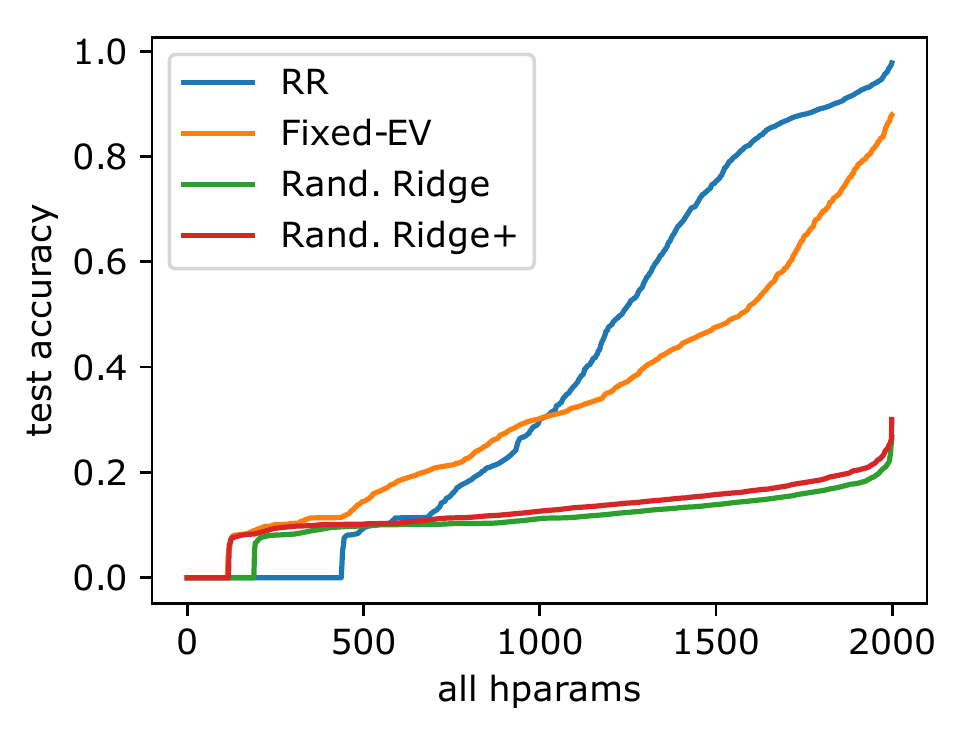}
    \caption{\textbf{Left}: Test and training accuracy on MNIST. Our hyperparameters for this experiment were: $S=236$, $\alpha=0.00264$, $\text{LR}_x~=0.000510$, $\text{LR}_\lambda=4.34e^{-6}$, $\text{batch size}=2236$. \textbf{Right}: We compare a hyperparameter sweep for approximate \method{} on MNIST with a simple ablation: Rather than updating the ridge (EVec), we set LR$_e$ to 0, i.e. keep following the original direction of the EVec. Shown are the runs that resulted in over $>60\%$ final test accuracy out of a hyper parameter sweep over 2000 random trials. We note that updating the ridge is absolutely crucial of obtaining high performance on MNIST - simply following the fixed eigenvectors with an otherwise unchanged \method{} algorithm never exceeds the performance of a linear classifier.}
    \label{fig:mnist_plots}
\end{figure*}

With fixed hyperparameters and initialization, the order in which the digit classes are learned changes according to the fingerprint. This is seen in Fig~\ref{fig:mnist_diversity}. The ridges with low indices (i.e. EVecs with very negative curvature) correspond to learning `0' and `1' initially, the intermediate ridges correspond to learning `2' and `3' first, and the ridges at the upper end of the spectrum we sampled (ie. $>30$) correspond to learning features for the digit ``8''.

\begin{figure*}[ht]
    \centering
    \includegraphics[width=\textwidth]{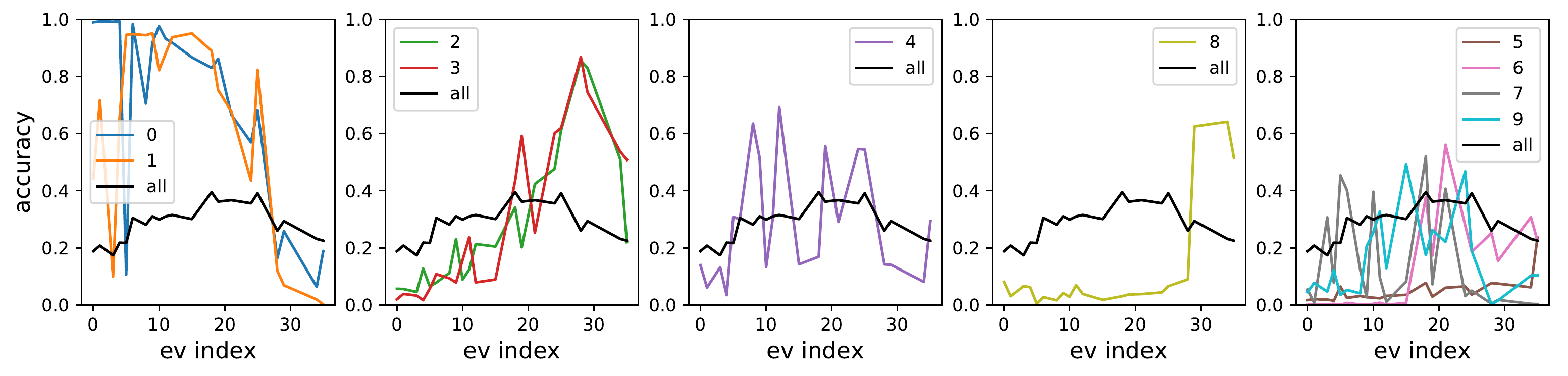}
    \includegraphics[width=\textwidth]{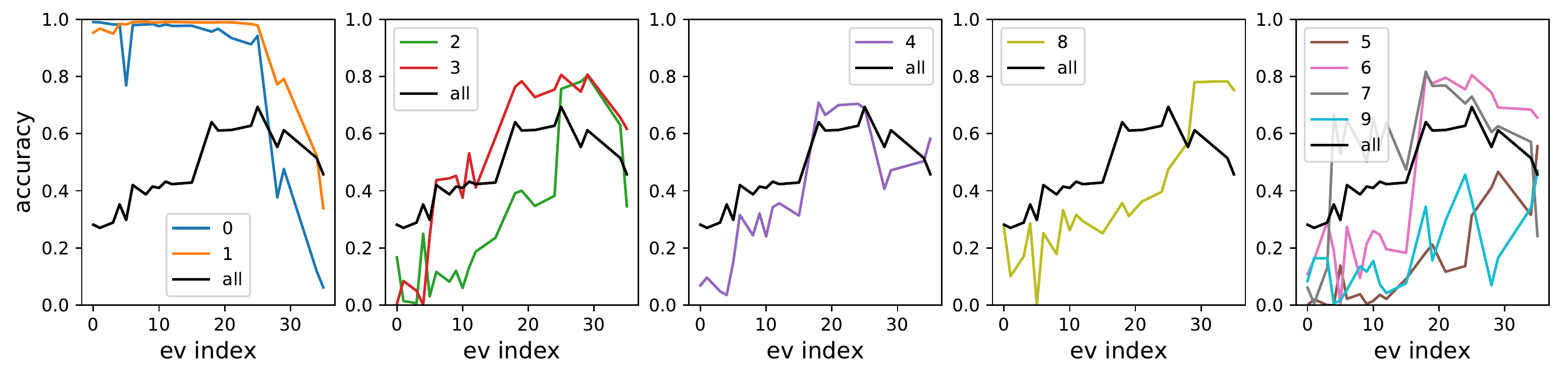}
    \caption{Class accuracy for different digits as a function of the index of the first ridge ($\psi[0]$), i.e the ranking of the EVal corresponding to the first EVec we follow. \textbf{Top}: Early in training -- average between $200$ and $600$ steps. \textbf{Bottom}: Later in training, averaged between $4000:5000$ steps. The architecture is the same MLP as in Figure~\ref{fig:mnist_plots}, but the hyperparameters are: $S=1$, $\alpha=0.000232$, $\text{LR}_x=3.20e^{-6}$, $LR_\lambda=0.00055$, $\text{batch size}= 2824$.}
    \label{fig:mnist_diversity}
\end{figure*}

\paragraph{\method{} for Zero-Shot Coordination} We test \accmethod{} as described in Sec. \ref{subsection:zeroshotrr} on multi-agent learning using the lever coordination game from~\citep{hu2020other}. The goal is to maximize the expected reward $J$ when playing a matrix game with a stranger. On each turn, the two players individually choose from one of ten levers. They get zero reward if they selected different levers and they get the payoff associated with the chosen lever if their choices matched. Importantly, not all levers have the same payoff. In the original version, nine of the ten levers payed $1$ and one paid $.9$. In our more difficult version, seven of the ten levers pay $1$, two `partial coordination' levers pay $0.8$, and one lever uniquely pays $0.6$. In self-play, the optimal choice is to pick one of the seven high-paying levers. However, since there are seven equivalent options, this will fail in zero-shot coordination. Instead, the optimal choice is to pick a lever which obtains the highest expected payoff when paired with any of the equivalent policies. Like in the RL setting, we use a BFS version of exact \accmethod{} and the MIS is found by optimizing for high entropy and low gradient.

\begin{figure*}[h]
    \centering
    \includegraphics[width=\linewidth]{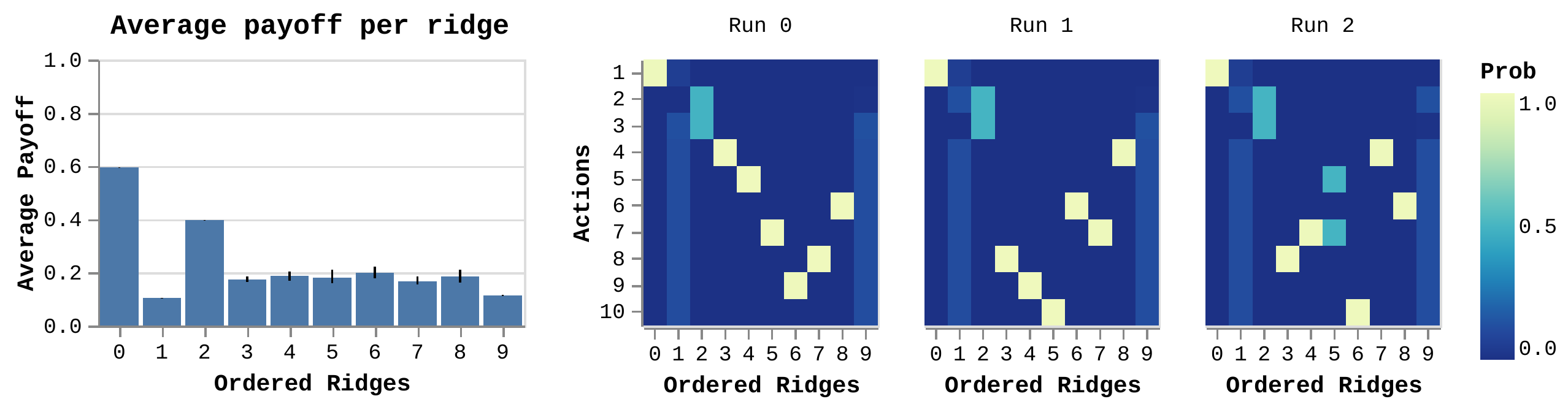}
    \caption{Zero-Shot Coordination: On the left, we see the average payoff per ridge over 25 runs, repeated five times to yield error estimates. As expected, the highest payoff is $0.6$ and it occurs when both agents find the symmetry breaking solution, even though that solution yields the lowest payoff in self-play. On the right, we see the results of three randomly chosen runs where each square is a probability that \emph{one} of the two agents select that action. We verified that the agents in each run and each ridge agree on the greedy action.}
    \label{fig:matrix_game}
\end{figure*}

We see in Fig~\ref{fig:matrix_game} that \accmethod{} is able to find each solution type: the self-play choices that do poorly when playing with others (ridges 1, 3-5, 7-9), the `partial coordination' that yield lower reward overall (ridges 2 and 6), and the ideal coordinated strategy (ridge 0). Note that based on this protocol the agent will chose the $0.6$ lever, achieving perfect zero-shot coordination. To run the zero-shot coordination experiment, see the notebook at \url{https://bit.ly/308j2uQ}.



\paragraph{\method{} for Out of Distribution Generalization} We test our extension of \accmethod{} from Sec \ref{subsec:OOD} on OOD generalization using Colored MNIST \citep{arjovsky2019invariant}. Following Sec~\ref{subsection:oodg}, for each of $\xi_{1, 2}$ in $\Xi$, as well as test environment $\xi_3$, $x_i \in D_{\xi_k}$ are drawn from disjoint subsets of MNIST \citep{lecun2010mnist} s.t. $|D_{\xi_{1, 2}}| = 25000$ and $|D_{\xi_{3}}| = 10000$. Further, each has an environment specific $p_{\xi_k}$ which informs $\mathcal{D}_{\xi_k}$ as follows: For each $x_i \in D_{\xi_k}$, first assign a preliminary binary label $\tilde{y_i}$ to $x_i$ based on the digit -- $\tilde{y_i} = 0$ for $x_i \in [0,4]$ and $\tilde{y_i} = 1$ for $x_i \in [5,9]$. The actual label $y_i$ is $\tilde{y_i}$ but flipped with probability $.25$. Then, sample color id $z_i$ by flipping $y_i$ with probability $p_{\xi_k}$, where $p_{\xi_1}=0.2$, $p_{\xi_2}=0.1$, and $p_{\xi_3}=0.9$. Finally, color the image red if $z_i = 1$ or green if $z_i = 0$.  Practically, we first optimize $\mathcal{L}_1+\mathcal{L}_2$ to find the MIS, resampling the DNN when optimization fails to obtain a low loss.

\begin{wraptable}{r}{5.5cm}
\vspace{-5mm}
\begin{tabular}{c|c|c}
Method & Train Acc & Test Acc \\
\toprule
\textbf{\accmethod{}} & $65.5 \pm 1.68$ & $58.4 \pm 2.41$ \\
ERM & $87.4 \pm 1.70$ & $17.8 \pm 1.33$ \\
IRM & $69.7 \pm .710$ & $65.7 \pm 1.42$ \\
\toprule
Chance & $50$ & $50$ \\
Optimal & $75$ & $75$ \\
\end{tabular}
\vspace{-1mm}
\caption{\textbf{Colored MNIST}: Accuracies on a $95\%$ confidence interval. \accmethod{} is in line with causal solutions.}
\label{table:colored-mnist}
\vspace{-3mm}
\end{wraptable} 

Chance is $50\%$. The optimal score is $75\%$ on train \emph{and} test. Fitting a neural network with ERM and SGD yields around $87\%$ on train and $18\%$ on test because it only finds the spurious color correlation. Methods which instead seek the more causal digit explanation achieve about $66\%-69\%$\cite{arjovsky2019invariant,ahuja2020invariant,krueger2020outofdistribution}. As shown in Table~\ref{table:colored-mnist}, our results over $30$ runs achieve a high 
after nine steps along the ridge of $65.5\% \pm 1.68$ on train and $58.4\% \pm 2.41$ on test. Our results are clearly both above chance and in line with models that find the causal explanation rather than the spurious correlative one. To run the out of distribution generalization experiment, see the notebook at \url{https://bit.ly/3gWeFsH}. See Fig~\ref{fig:colored_mnist} in the Appendix for additional results.

\section{Discussion}

We have introduced \method{}, a novel method for finding specific types of solutions, which shows promising results in a diverse set of problems. In some ways, this paper itself can be thought of as the result of running the breadth-first version of \accmethod{} - a set of early explorations into different directions of high curvature, which one day will hopefully lead to SotA results, novel insights and solutions to real world problems.
However, there is clearly a long way to go. Scaling \accmethod{} to more difficult problems will require a way to deal with the noise and stochasticity of these settings. It will also require more efficient ways to compute eigenvalues and eigenvectors far from the extreme points of the spectrum, as well as better understanding of how to follow them robustly. Finally, \accmethod{} hints at conceptual connections between generalization in supervised learning and zero-shot coordination, which we are just beginning to understand. Clearly, symmetries and invariances in the Hessian play a crucial, but under-explored, role in this connection.

\section*{Acknowledgements}
We'd like to thank Brendan Shillingford, Martin Arjovsky, Niladri Chatterji, Ishaan Gulrajani and C. Daniel Freeman for providing feedback on the manuscript.

\section*{Broader Impact}

We believe our method is the first to propose following the eigenvectors of the Hessian to optimize in the parameter space to train neural networks. This provides a stark contrast to SGD as commonly used across a broad spectrum of applications. Most specifically, it allows us to seek a variety of solutions more easily. Given how strong DNNs are as function approximators, algorithms that enable more structured exploration of the range of solutions are more likely to find those that are semantically aligned with what humans care about.

In our view, the most significant advantage of that is the possibility that we could discover the minima that are not `shortcut solutions' \citep{Geirhos2020a} like texture but rather generalizable solutions like shape \citep{DBLP:journals/corr/abs-1811-12231}. The texture and shape biases are just one of many problematic solution tradeoffs that we are trying to address. This also holds for non-causal/causal solutions (the non-causal or correlative solution patterns are much easier to find) as well as concerns around learned biases that we see in applied areas across machine learning. All of these could in principle be partially addressed by our method.

Furthermore, while SGD has been optimized over decades and is extremely effective, there is no guarantee that  \method{} will ever become a competitive optimizer. However, maybe this is simply an instance of the `no-free-lunch' theorem \citep{nfl} - we cannot expect to find diverse solutions in science unless we are willing to take a risk by not following the locally greedy path. Still, we are committed to making this journey as resource-and time efficient as possible, making our code available and testing the method on toy environments are important measures in this direction. 

\bibliographystyle{abbrv}
\bibliography{refs}

\newpage
\appendix

\section*{Appendix: Ridge Rider}

\section{Related Work}
The common approach to finding a specific type of solution to a DNN optimization problem is to modify SGD via a range of algorithmic approaches to initialization, update rules, learning rate, and so forth~\cite{DBLP:journals/corr/ChaudhariCSL16,DBLP:journals/corr/KeskarMNST16,DBLP:journals/corr/NeyshaburSS15,DBLP:journals/corr/NeyshaburBMS17}. By contrast, \accmethod{} does not follow the gradient at all - instead of pursuing the locally \emph{greedy} direction that SGD seeks, we follow an eigenvector of the Hessian which allows us to directly control for curvature. 
\citet{wang2019solving} also adjusts the solutions found when training neural networks by taking advantage of the Hessian to stay on `ridges'. However, they focus on minimax optimization such as in the GAN \cite{NIPS2014_5423} setting. This difference in motivation leads them to an algorithm that looks like Gradient Descent Ascent but with a correction term using the Hessian that keeps the minimax problem from veering into problematic areas with respect to convergence. This is different from our approach, which moves in the direction of the Hessian's eigenvectors instead of gradients. Eigen Vector Descent (EVD, \cite{evd}) proposes to update neural network parameters in the direction of an eigenvector of the Hessian, by doing line search on each individual eigenvector at each step and taking the best value. This can be seen as a myopic version of our method, which greedily updates a single policy. By contrast, we do deep exploration of each eigenvector and maintain a set of candidate solutions. 

Finally, there are also other optimization approaches that do not rely on gradients such as BFGS and Quasi-Newton. They efficiently optimize towards the nearest solution, sometimes by using the Hessian, but do not do what we are proposing of using the eigenvectors as directions to follow. Rather, they use variants of line search with constraints that include the Hessian.

Motivation for \accmethod{} stems from a number of different directions. One is in games and specifically the zero-shot coordination setting. \citet{hu2020other} presents an algorithm that achieves strong results on the Hanabi benchmark \cite{DBLP:journals/corr/abs-1902-00506} including human-AI coordination, but require problem specific knowledge of the game's symmetries. These symmetries correspond to the arbitrary labelings of the
particular state/action space that leave trajectories unchanged up to those labelings. In contrast, \accmethod{} discovers these symmetries automatically by exploiting the connection between equivalent solutions and repeated Eigenvalues, as we demonstrate empirically in Sec~\ref{sec:experiments}.


Another motivating direction is avoiding the `shortcut' solutions often found by DNNs. These are often subpar for downstream tasks~\cite{Geirhos2020a,NIPS2018_7982,benaim2020speednet,Doersch_2015,Han_2019}. \citet{DBLP:journals/corr/abs-1811-12231} tries to address this concern by adjusting the dataset to be less amenable to the texture shortcut and more reliant on shape. In the causal paradigm, \citet{arjovsky2019invariant} assumes access to differing environments and then adjusts the SGD loss landscape to account for all environments simultaneously via a gradient norm penalty. We differ by looking for structured solutions by following the curvature of the Hessian.

In RL, many approaches make use of augmented loss functions to aid exploration, which are subsequently optimized with SGD. These include having a term boosting diversity with respect to other agents \cite{Lehman08exploitingopen, lehmannovelty, novelty, diversitydriven,  parkerholder2020effective, Mouret2015IlluminatingSS, qdnature} or measuring `surprise' in the environment \cite{mohamed2017, eysenbach2018diversity, hashexplore, pathak_og, Raileanu2020RIDE, burda2018exploration}. Rather than shifting or augmenting the loss landscape in this way, we gain diversity through structured exploration with eigenvectors of the Hessian. Finally, maximum entropy objectives are also a popular way to boost exploration in RL \cite{sac, eysenbach2019maxent}. However, this is typically combined with SGD rather than used as an initialization for \accmethod{} as we propose.

\newpage
\section{Additional Experimental Results}
\label{sec:expdetails}

In this section we include some addiitonal results from the diversity in RL experiments.

\textbf{Intuition and Ablations} To illustrate the structured exploration of \accmethod{}, we include a visualization of the optimization path from the MIS. On the left hand side of Fig~\ref{fig:rr_rl_ablation} we show the path along two ridges for the tree shown earlier alongside the main results (Fig~\ref{fig:tree_results}). Areas of high reward are in dark blue. The ridges (green and light blue) both correspond to distinct positive solutions. The policies have six dimensions, but we project them into two dimensions by forming a basis with the two ridges.
Observe that the two ridges are initially orthogonal, following the $(x,y)$ axes. Conversely, we also show two runs of GD from the same initialization, each of which find the same solution and are indistinguishable in parameter space.

\vspace{-2mm}
\begin{figure}[H]
        \begin{minipage}{0.99\textwidth}
         \includegraphics[width=0.3\linewidth]{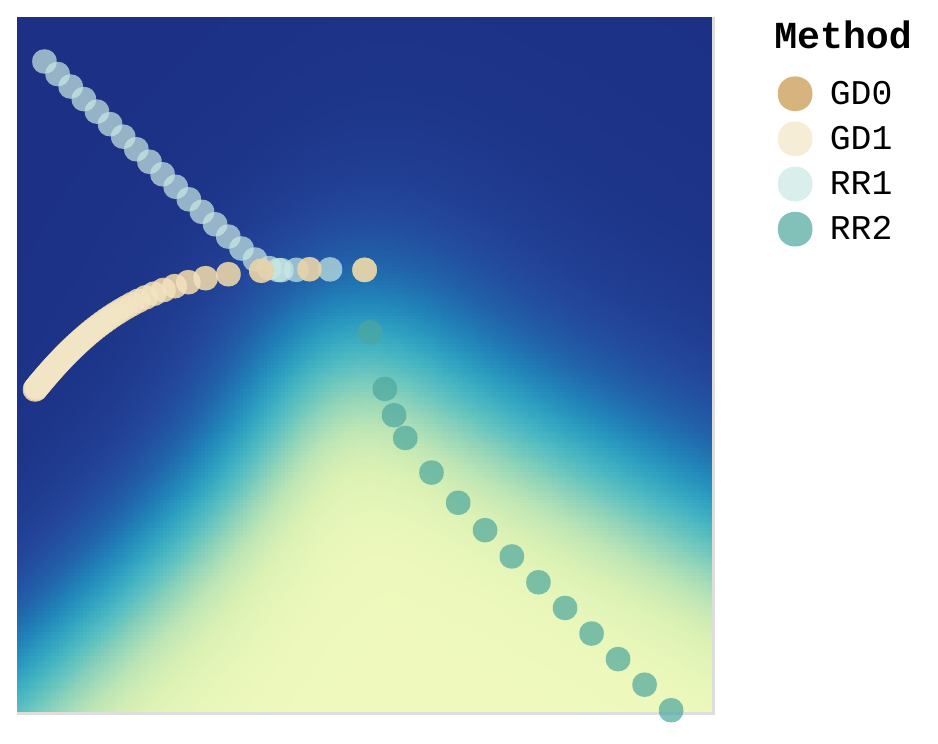} 
     \includegraphics[width=0.7\linewidth]{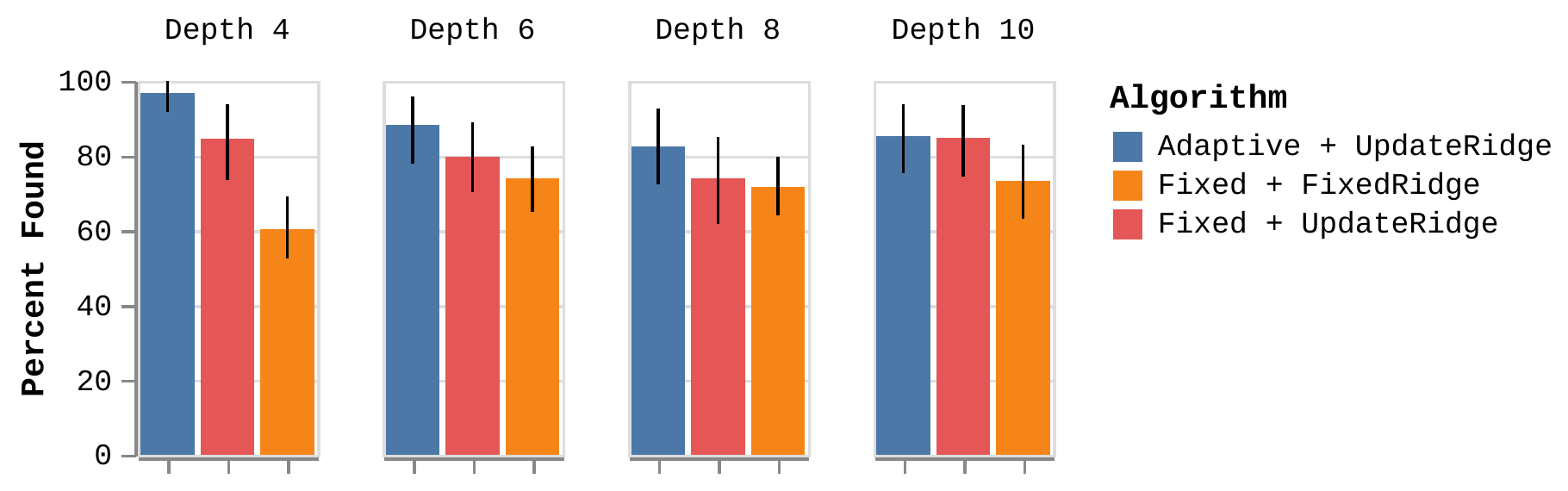}
        \end{minipage}
        \caption{\textbf{Left}: Example optimization paths. Dark is higher reward. All methods start at the same saddle. The deterministic GDs follow the same trajectory while the \method{}s follow different paths. \textbf{Right}: The percentage of solutions found per algorithm, collated by tree depth. Trees at each depth are randomly generated $20$ times to produce error estimates shown.}
        \label{fig:rr_rl_ablation}
\end{figure}
\vspace{-4mm}

We also must consider whether we need to adapt $\alpha$, the size of the step along the ridge. In larger scale settings this may not be possible, so we now consider the two following ablations:
\begin{enumerate}
    \item \textit{Fixed + UpdateRidge}: Here we continue to update the ridge, but use a fixed $\alpha$.
    \item \textit{Fixed + FixedRidge}: We not only keep a fixed $\alpha$ but also do not update the ridge. Thus, we take repeated small steps along the original ridge until we meet the $\mathrm{EndRide}$ condition.
\end{enumerate}

In both cases we use $\alpha=0.1$, the same value used for the Gradient Descent and Random Vector baselines in Fig. \ref{fig:tree_results}. As we see in Fig. \ref{fig:rr_rl_ablation}, as we impose greater rigidity, the performance declines. However, the results for the fixed ridge are still stronger than any of the baselines from Fig. \ref{fig:tree_results}, showing the power of the original ridges in finding diverse solutions. 

\textbf{Beyond Exact RL} Next we open the door to scaling our method to the deep RL paradigm. For on-policy methods we can compute $\mathcal{H}$ accurately via the DiCE operator \cite{dice, loaded_dice}. As a first demonstration of this, we consider the same tree problem shown in Fig. \ref{fig:tree_results}, with the same tabular/linear policy. However, instead of computing the loss as an expectation with full knowledge of the tree, we sample trajectories by taking stochastic actions of the policy in the environment. We use the same maximum budget of $T=10^5$, but using $100$ samples to compute the loss function. For \accmethod{} we alternate between updating the policy by stepping along the ridge and updating the value function (with SGD), while for the baseline we optimize a joint objective. Both \method{} and the baseline use the Loaded DiCE \cite{loaded_dice} loss function for the policy. For \method{}, we use the adaptive $\alpha$ from the exact setting. The results are shown in Fig. \ref{fig:dice_rl}.

\begin{wrapfigure}{r}{0.45\textwidth} 
        \vspace{-9mm}
        \begin{minipage}{0.45\textwidth}
        \centering{\includegraphics[width=0.9\linewidth]{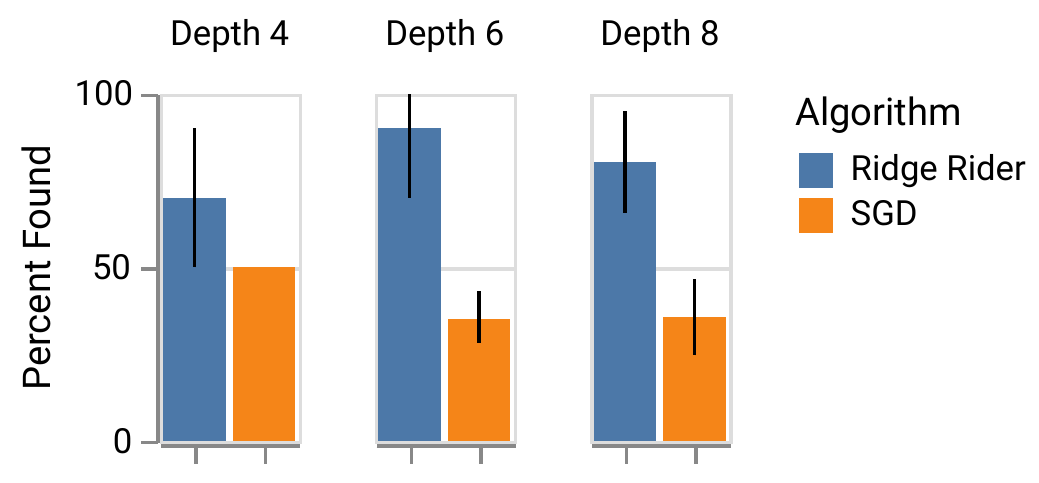}} 
        \end{minipage}
        \caption{The percentage of solutions found per algorithm, collated by tree depth. Trees at each depth are randomly generated $10$ times to produce error estimates shown.}
        \label{fig:dice_rl}
        \vspace{-5mm}
\end{wrapfigure}

For tree depths of four and six, \method{} significantly outperforms the vanilla actor critic baseline, with gains also coming on larger trees. It is important to caveat that these results remain only indicative of larger scale performance, given the toy nature of the problem. However, it does show it is possible to follow ridges when the Hessian is computed with samples, thus, we believe it demonstrates the \emph{potential} to scale \accmethod{} to the problems we ultimately wish to solve. G

\newpage
\section{Implementation Details}
\label{sec:exp_details}
In the following subsections, we provide implementation details along with pseudo code for the approximate version of \method{} (used in MNIST experiments) and the extensions to Zero-Shot coordination and Colored MNIST, respectively.

\subsection{Approximate \accmethod{}}
\begin{algorithm}[H]
\SetAlgoLined
\scriptsize
\caption{Scalable \accmethod{}}
\begin{algorithmic}[1]

\STATE \textbf{Input:} $N$ number of ridges; $R$ maximum index of ridge considered; $T$ max iterations per ridge; $S$ inner steps; LR$_{e/\lambda}$, learning rate for EVec and EVal 
\STATE \textbf{Initialize:} Sample neural network $\theta\sim N(0,\epsilon)$ small value, i.e. near saddle.
\STATE // Find a ridge. 
\STATE $e, \lambda, r =\mathrm{GetRidges}(\theta)$: Sample random integer $r \in [1, R]$ and mini-batch ($\bm{x}, \bm{y}$),
\STATE then use the power method to find the $r$-th most negative Eval $\lambda$ and its EVec ${e}$.
\STATE Optionally: Compute gradient $\bm{g} = \frac{\partial }{\partial {\theta}} L\left(\bm{y}, f_{{\theta}}(\bm{x})\right)$ and set ridge ${e} = \text{sign}({e} \cdot \bm{g}) {e}$.
\STATE archive $= [\{\theta^{[r]}, e, \lambda \}]$
\FOR{$1$  to  N}
   \STATE ${\theta^\Psi, \lambda, e} =$ archive.pop(0) // $\mathrm{ChooseFromArchive}$ is trivial - there is only one entry
  \STATE // Follow the ridge. 
  \WHILE{True}
    \STATE ${\theta}_{old} = {\theta^\Psi}$
    \STATE ${\theta^\Psi} = {\theta} - lr_{{\theta}}~ {e} $
    \STATE $\lambda = \lambda +  {e}^T (\mathcal{H}({\theta^\Psi})- \mathcal{H}({\theta}_{old})) {e} $
    \STATE $\mathrm{UpdateRidge}$ // Update eigenvalue and eigenvector.
    \STATE Sample a mini-batch $\bm{x}, \bm{y}$.
    \STATE // Starting from $\lambda$,$e$ use gradient descent to obtain updated values:
    \FOR {$1$ to S}
    \STATE $\mathcal{L}(e, \lambda | \theta^\Psi) =  (\norm{(1/\lambda) \mathcal{H}({\theta^\Psi}){e} / \norm{{e}}_2 - {e} / \norm{{e}}_2}_2^2)$
      \STATE $\lambda = \lambda -  \text{LR}_{\lambda} \frac{\partial }{\partial \lambda}\mathcal{L}(e, \lambda | \theta^\Psi)$
      \STATE $e = e -  \text{LR}_e \frac{\partial }{\partial e} \mathcal{L}(e, \lambda | \theta^\Psi)$
    \ENDFOR
    \IF{$\mathrm{EndRide}(\theta^\Psi, e, \lambda)$}
    \STATE \textbf{break}
    \ENDIF
  \ENDWHILE
  \STATE $e, \lambda, r = \mathrm{GetRidges}(\theta^\Psi)$
  \STATE archive =$ [\{\theta^{\Psi.\text{append}(r)}, e, \lambda \}]$
\ENDFOR
\STATE \textbf{return} $\theta^{\Psi}$
\end{algorithmic}
\label{algorithm:approximaterr}
\end{algorithm}

\subsection{Multi-Agent Zero-Shot Coordination}
\begin{algorithm}[H]
\scriptsize
\SetAlgoLined
\caption{Zero-Shot Coordination}
\begin{algorithmic}[1]

\STATE \textbf{Input:} $N$ independent runs; List $\mathbf{\Psi}=\{\Psi_1, ... \Psi_N \}$ fingerprints considered; $T$ max iterations per ridge; Learning rate $\alpha$.
\STATE \textbf{Initialize:} Array solutions = [][], best\_score = -1.
\FOR{$i \in 1$  to  N} 
  \STATE \textbf{Initialize:} Policy $\theta \sim N(0,\epsilon)$. \\
  // Get the Maximally Invariant Saddle (min gradient norm, max entropy).
  \STATE MIS $= \mathrm{GetMIS}(\theta)$ 
  \STATE $s = \mathrm{RidgeRiding}(\text{MIS},\mathbf{\Psi}$, T, $\alpha$) //Run exact version of \method{}\\
  // For each $\Psi$, select $\theta$ with highest reward in self-play. Note, multiple $\theta$ per $\Psi$ correspond to $\pm$ EVec directions. 
  \FOR{$k \in 1$ to $\text{len}(\mathbf{\Psi})$} 
      \STATE  solutions$[i][k] = \text{argmax}_{\theta~\text{s.t.}~\theta^{\Psi == \mathbf{\Psi}[k]}} J(\theta)$
    \ENDFOR
  \ENDFOR
\FOR{$k  \in 1$ to $\text{len}(\mathbf{\Psi})$} 
\STATE \textbf{Initialize:} average\_score = 0 // Average cross-play score for fingerprint $\Psi$
\FOR{$i  \in  \{1, \ldots, N \}$} 
\FOR{$j  \in  \{1, \ldots, N \}$} 
\STATE average\_score $\pluseq J(\text{solutions}[i][k]_1, \text{solutions}[j][k]_2)/N^2$  
\ENDFOR
\ENDFOR
\IF{average\_score > best\_score  }
\STATE best\_score = average\_score
\STATE $\theta^*$ = \text{solutions}[0][k]
\ENDIF
\ENDFOR
\STATE \textbf{return} $\theta^*$
\end{algorithmic}
\label{algorithm:zero-shot}
\end{algorithm}

\subsection{Colored MNIST}
\label{sec:ood}

\begin{algorithm}[H]
\scriptsize
\SetAlgoLined
\caption{Colored MNIST}
\begin{algorithmic}[1]

\STATE \textbf{Input:} Training Environments $\xi_1$, $\xi_2$; Inner steps S; Ridge steps N; Pre-Training Steps H; Loss hyperparameters $\mathbf{\beta}$; Number of featurizer weights $f$; Learning rate for featurizer and EVec LR$_{f/x}$; Learning rate for \method{} $\alpha$; Learning rates for finding MIS in pre-training $\mathbf{\gamma}$.
\STATE \textbf{Initialize:} Neural network $\theta \sim N(0, \epsilon)^n$ to small random values, ie. near saddle; candidate common EVec $e_r \sim N(0, \epsilon)^r$ where $r = n - f$. \\
// Split weights into featurizer and \method{} space.
\STATE $\theta_f =\theta[0:f]$
\STATE $\theta_r =\theta[f: n ]$
 \STATE // Find the Maximally Invariant Saddle and initial common EVec.
   \FOR{$1$  to  H} 
   \STATE $\mathcal{L}_{1}(\theta_{f}, e_r | \theta_r) =\sum_{i \in 1,2} \big( -\beta_1 \mathrm{C}(\mathcal{H}^i_r e_r, e_r) - \beta_2 e_r \mathcal{H}^i_r e_r + \beta_3 {L}_{\xi_i}(\theta_{f}| \theta_r)\big) + \frac{|e_r ( \mathcal{H}^1_r - \mathcal{H}^2_r) e_r   |}{|e_r|^2}$
   \STATE $\mathcal{L}_{2}(\theta_{f} | \theta_r) =\mathbb{E}_{u_r\sim\Theta_r} \beta_4 \frac{|\mathrm{C}(\mathcal{H}^1_r u_r, u_r)^2 - \mathrm{C}(\mathcal{H}^2_r u_r, u_r)^2 |}{\mathrm{C}(\mathcal{H}^1_r u_r, u_r)^2 +\beta_5 \mathrm{C}(\mathcal{H}^1_r u_r, u_r)^2}+ \frac{|u_r ( \mathcal{H}^1_r - \mathcal{H}^2_r) u_r|}{|u_r  \mathcal{H}^1_r u_r| + |u_r \mathcal{H}^2_r  u_r|}$
   \STATE // $\mathcal{H}^{\xi}_r = \nabla^2_\theta \mathcal{L}_{\xi}$ and $C$ is the correlation (normalized dot product)
  \STATE  $\theta_f =\theta_f - \gamma_0  \frac{\partial }{\partial \theta_f} (L_1 + L_2)$  
  \STATE  $e_r = e_r - \gamma_1 \frac{\partial }{\partial e_r} (L_1)$  
  \ENDFOR
  \FOR{$1$ to N}
    \STATE // Follow the ridge.
    \STATE $\theta_r =\theta_r - \alpha x$ \\
    // Update EVec and featurizer.
    \FOR {$1$ to S}
        \STATE  $\theta_f =\theta_f -  \text{LR}_f  \frac{\partial }{\partial \theta_f} (L_1 + L_2)$  
        \STATE  $e_r = e_r -\text{LR}_x \frac{\partial }{\partial e_r} (L_1)$  
    \ENDFOR
  \ENDFOR

\end{algorithmic}
\label{algorithm:cmnistrr}
\end{algorithm}

\begin{figure*}[ht]
    \centering
    \includegraphics[width=0.7\linewidth]{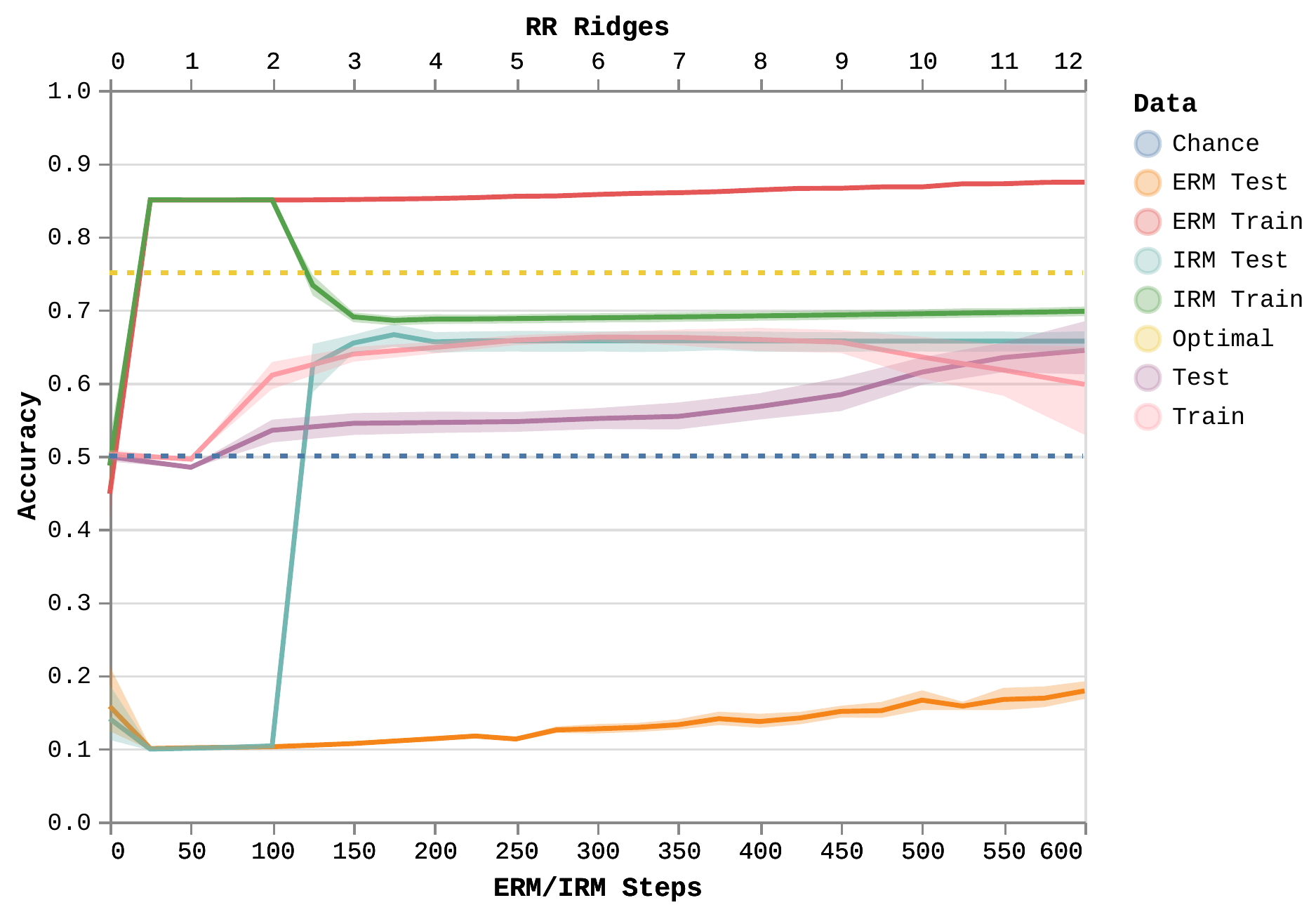}
    \caption{Curves showing the training and test accuracy of IRM (Invariant Risk Minimization), ERM (Empirical Risk Minimization), and \method{} on Colored MNIST. Note that the bottom x-axis is steps for ERM/IRM and the top x-axis is steps for \method{}.}
    \label{fig:colored_mnist}
\end{figure*}

\newpage

\section{Theoretical Results}

\subsection{Behavior of gradient descent near a saddle point}
\label{app:gd_dynamics}
We will illustrate how gradient descent dynamics near a saddle point moves towards the most negative eigenvector of the Hessian via two different derivations. These are not novel results but provided as an illustration of a well known fact.

First, let $\theta_0$ be a saddle point of $\mathcal{L}(\theta)$, and consider $T$ steps of gradient descent $\theta_1, ..., \theta_T$. Let $\mathcal{H}(\theta_0)$ be the Hessian of $\mathcal{L}$ at $\theta_0$. We will use the first-order Taylor expansion, $\nabla_\theta \mathcal{L}(\theta_t) = H (\theta_t - \theta_0) + o(\epsilon^2)$ ignoring the error term to approximate the gradient close to $\theta_0$.

We can decompose $\theta_t - \theta_0$ into the basis of eigenvectors of $\mathcal{H}(\theta_0)$:  $\theta_t - \theta_0 = \sum_i { a_{i,t} e_i(\theta_0)}$ where $\{ e_i(\theta_0) \}$ are the eigenvectors of $\mathcal{H}(\theta_0)$. After one step of gradient descent with learning rate $\alpha$, 

\begin{align}
    \theta_{t+1} & = \theta_0 + \sum_i {a_{i,t} e_i(\theta_0)} - \alpha \sum_i {\lambda_i a_{i,t} e_i(\theta_0)}\\ 
    & = x_0 + \sum_i {(1 - \alpha \lambda_i(\theta_0)) a_{i,t} e_i(\theta_0)}
\end{align}

i.e. $a_{i,t+1} = (1 - \alpha \lambda_i(\theta_0)) a_{i,t}$

It follows by simple induction that if $T$ isn't too large so the displacement $\theta_T  - \theta_0$ is still small, $a_{i,T} = (1 - \alpha \lambda_i)^T$. In other words, the component of $\theta_1-\theta_0$ corresponding to more negative eigenvalues of $\mathcal{H}(\theta_0)$ will be amplified relative to less negative eigenvalues by a ratio that grows exponentially in $T$. 

\vspace{0.5cm}

In the limit of small step sizes, we can also consider the differential limit of approximate (up to first order terms as defined above) gradient descent dynamics $$\frac{d\theta}{dt} = -\alpha \mathcal{H}(\theta_0)\theta,$$ assuming the saddle is at $\theta_0 = 0$ wlog. The solution to this system of equations is $\theta(t) = \theta(0) \exp( -\alpha \mathcal{H}(\theta_0) t )$. If we write the eigendecomposition of $\mathcal{H}(\theta_0)$, $\mathcal{H}(\theta_0) = Q \Lambda Q^{-1}$, then $\exp(-\mathcal{H}(\theta_0)) = - Q \exp(\Lambda) Q^{-1}$. So if $\theta(0) = \sum_i a_i e_i(\theta_0)$, then $\theta(t)= \sum_i e^{-\alpha \lambda_i(\theta_0) t} a_i e_i(\theta_0)$.

\subsection{Structural properties of the eigenvalues and eigenvectors of Smooth functions}\label{section::structural_eigen}

\begin{definition}
We say a function $f:\mathbb{R}^d \rightarrow \mathbb{R}$ is $\beta-$smooth if for all pairs of points $\theta, \theta' \in \mathbb{R}^d$:
\begin{equation*}
    | \nabla_\theta^2 f(\theta) -  \nabla_\theta^2 f(\theta') | \leq \beta \| \theta - \theta'\|
\end{equation*}

\end{definition}

We show that for $\beta-$smooth functions $f$, the $i-$th eigenvalue function:

\begin{equation}
    \lambda_i(\theta) = i-\text{th largest eigenvalue of }\nabla_\theta^2 f(\theta)
\end{equation}

And (an appropriate definition of) the $i-$th eigenvector function:

\begin{equation}
    e_i(\theta) = i-\text{th largest normalized eigenvector of }\nabla_\theta^2 f(\theta)
\end{equation}

are continuous functions of $\theta$. 

\begin{lemma}\label{lemma::beta_smooth_cont_eigenval}
If $f$ is $\beta-$smooth` $\lambda_i(\theta)$ is continuous. 
\end{lemma}

\begin{proof}
We show the result for the case $i=1$, the proof for $i \neq 1$ follows the same structure, albeit  making use of the more complex variational characterization of the $i-$th eigenvalue / eigenvector pair. Recall the variational formulation of $\lambda_1$:

\begin{equation}\label{equation::variational_characterization_eigenvalue}
    \lambda_1(\theta) = \max_{v \in \mathcal{S}_d(1) } v^\top \nabla_\theta^2 f(\theta) v 
\end{equation}

Let $\theta' = \theta + \Delta_\theta$. It is enough to show that:

\begin{equation*}
    \lim_{\Delta_\theta \rightarrow 0} \lambda_1(\theta') = \lambda_1(\theta)
\end{equation*}

Let $v_1 $ be a unit vector achieving the max in Equation \ref{equation::variational_characterization_eigenvalue} and let $v_1'$ be the maximizer for the corresponding variational equation for $\theta'$, then:

\begin{align*}
    |\lambda_1(\theta) -  v_1^\top \nabla^2_\theta f(\theta') v_1 |    &=  |v_1^\top \left( \nabla_\theta^2 f(\theta)- \nabla^2_\theta f(\theta')  \right) v_1 |\\
    &\stackrel{(i)}{\leq}\beta   \|\Delta_\theta   \|
\end{align*}

Inequality $(i)$ follows by $\beta-$smoothness. Similarly:

\begin{equation*}
    | \lambda_1(\theta') - (v_1') \nabla_\theta^2 f(\theta) v_1' | \leq \beta \| \Delta_\theta\|
\end{equation*}
Since by definition:
\begin{equation*}
\lambda_1(\theta) \geq (v_1')^\top \nabla_\theta^2 f(\theta) v_1'
\end{equation*}
And:
\begin{equation*}
    \lambda_1(\theta') \geq v_1^\top \nabla_\theta^2 f(\theta') v_1
\end{equation*}
We conclude that:
\begin{equation*}
    \lambda_1(\theta') \geq \lambda_1(\theta) - \beta\| \Delta_\theta\|
\end{equation*}
And:
\begin{equation*}
\lambda_1(\theta) \geq \lambda_1(\theta') - \beta\| \Delta_\theta\|
\end{equation*}
Consequently:
\begin{equation*}
    \lambda_1(\theta) + \beta \| \Delta_\theta \| \geq \lambda_1(\theta') \geq \lambda_1(\theta) - \beta \| \Delta_\theta\|
\end{equation*}
The result follows by taking the limit as $\Delta_\theta \rightarrow 0$.
\end{proof}

In fact the proof above shows even more:

\begin{proposition}\label{proposition::eigenvalue_lipschitz}
If $L$ is $\beta-$smooth and $\theta, \theta' \in \mathbb{R}^d$ then the eigenvalue function is $\beta$-Lipschitz: 
\begin{equation*}
    | \lambda_i(\theta') - \lambda_i(\theta) | \leq \beta \| \theta - \theta'\|
\end{equation*}
And therefore:
\begin{equation*}
    \| \nabla_\theta \lambda_i(\theta) \| \leq \beta 
\end{equation*}
\end{proposition}

\begin{proof}
    With the exact same sequence of steps as in the proof of Lemma \ref{lemma::beta_smooth_cont_eigenval}, we conclude:
    \begin{equation*}
         \lambda_i(\theta) + \beta \| \Delta_\theta \| \geq \lambda_i(\theta') \geq \lambda_i(\theta) - \beta \| \Delta_\theta\|
    \end{equation*}
    And:
        \begin{equation*}
         \lambda_i(\theta') + \beta \| \Delta_\theta \| \geq \lambda_i(\theta) \geq \lambda_i(\theta') - \beta \| \Delta_\theta\|
    \end{equation*}
The result follows. The gradient bound is an immediate consequence of the Lipschitz property of $\lambda_1(\cdot)$. 
\end{proof}

Let's define a canonical $i-$th eigenvector path $\ell : [ 0, \infty ) \rightarrow \mathbb{R}$ starting at $\theta$ as follows:
\begin{align*}
    \ell(0 ) &= \theta\\
    \frac{ \partial \ell(t)}{\partial t} &=  \begin{cases}
    e_i(\theta) & \text{if }\lim_{l \rightarrow t} \langle e_i(\theta), \frac{\partial \ell(l)}{\partial l}\ell(l) \rangle > 0\\
    -e_i(\theta) &\text{o.w.}
    \end{cases}
\end{align*}

We proceed to show that the curve traced by $\ell$ is continuous. 

\begin{lemma}\label{lemma::lipschitzness_eigenvector}
Let $\theta$ be such that $\lambda_{i-1}(\theta) - \lambda_{i}(\theta) = \Delta_{i-1} > 0$ and $\lambda_{i}(\theta) - \lambda_{i+1}(\theta) = \Delta_{i} > 0$. Then:
\begin{equation}
\min( \| e_i( \theta) - e_i(\theta') \|,  \| e_i( \theta) + e_i(\theta') \| ) \leq \sqrt{\frac{ 4\beta \| \theta - \theta'\|}{\min(\Delta_i, \Delta_{i-1})}}
\end{equation}
 For all $\theta'$ such that $\| \theta' - \theta\| \leq \frac{\min( \Delta_{i-1}, \Delta_{i}) }{4\beta}$
\end{lemma}

\begin{proof}
By Proposition \ref{proposition::eigenvalue_lipschitz}, for all $\theta'$ such that $\| \theta' - \theta\| \leq \frac{\min( \Delta_{i-1}, \Delta_{i}) }{4\beta}$:
\begin{equation}\label{equation::eigenvalue_gaps_distinct}
\max(     |\lambda_i(\theta') - \lambda_i(\theta) |  ,  |\lambda_{i-1}(\theta') - \lambda_{i-1}(\theta) |,  |\lambda_{i+1}(\theta') - \lambda_{i+1}(\theta) | )\leq \beta \|\theta' - \theta\| 
\end{equation}
In other words, it follows that:
\begin{equation*}
    \lambda_{i+1}(\theta) , \lambda_{i+1}(\theta') < \lambda_i(\theta), \lambda_i(\theta') < \lambda_{i-1}(\theta), \lambda_{i-1}(\theta')
\end{equation*}
Where this sequence of inequalities implies for example that $\lambda_{i+1}(\theta')  < \lambda_i(\theta)$. 


Let $H = \nabla_\theta^2 f(\theta)$ and $H'= \nabla_\theta^2 f(\theta')$. Let $\Delta = H - H'$. By $\beta-$smoothness we know that $\| \Delta \| \leq \beta \| \theta - \theta'\|$.

Again for simplicity we restrict ourselves to the case $i=1$. The argument for $i \neq 1$ uses the same basic ingredients, but takes into account the more complex variational characterization of the $i-$th eigenvector. 

W.l.o.g. define $e_1(\theta)$ and $e_1(\theta')$ such that $\langle e_1(\theta), e_1(\theta') \rangle = \alpha \geq 0$. We can write $e_1(\theta') = \alpha e_1(\theta) + \left(\sqrt{1-\alpha^2 }\right)v$ with $\| v\|_2 = 1$ and $\langle v, e_1(\theta) \rangle= 0$. Notice that $\| e_1(\theta) - e_1(\theta') \| = \| (1-\alpha) e_1(\theta)   - \left(\sqrt{1-\alpha^2 }\right)v \| \leq (1-\alpha) + \sqrt{ 1-\alpha^2}$. We now show $\alpha$ is close to $1$. Recall:

\begin{equation*}
    e_1(\theta) = \arg\max_{v \in \mathbb{S}_d} v^\top H v
\end{equation*}
and
\begin{equation*}
    e_1(\theta') = \arg\max_{v' \in \mathbb{S}_d} (v')^\top H' v'
\end{equation*}
Equation \ref{equation::eigenvalue_gaps_distinct} implies $\lambda_1(\theta) \geq \lambda_2(\theta')  + \Delta_1 - \beta \| \theta - \theta'\|$ and $\lambda_1(\theta') \geq \lambda_2(\theta)  + \Delta_1 - \beta \| \theta - \theta'\|$. The following inequalities hold:

\begin{align}
\lambda_1(\theta') &\geq e_1(\theta)^\top    H' e_1(\theta) \\
&= e_1(\theta)^\top [ H - \Delta] e_1(\theta) \notag\\
&= e_1(\theta)^\top  H e_1(\theta)  - e_1(\theta)^\top\Delta e_1(\theta) \notag \\
&= \lambda_1(\theta) - e_1(\theta)^\top \Delta e_1(\theta) \notag\\
&\geq \lambda_1(\theta) - \| \Delta \| \notag\\
&\geq \lambda_1(\theta) - \beta \| \theta -\theta'\| \label{equation::lower_bound_1}
\end{align}

Write $e_1(\theta) = \sum_{i} \alpha_i e_i(\theta')$ with $\sum_i \alpha_i^2 = 1$. Notice that:
\begin{equation*}
    e_1(\theta)^\top H' e_1(\theta) = \sum_i \alpha_i^2 \lambda_i(\theta')
\end{equation*}
Since $\lambda_i(\theta') < \lambda_1(\theta) - \frac{3\Delta_1}{4}$ for all $i > 1$:
\begin{align}
    \sum_i \alpha_i^2 \lambda_i(\theta') &\leq \alpha_1^2 \lambda_1(\theta') +\left( \sum_{i=2}^d \alpha_i^2\right) \left(\lambda_1(\theta) - \Delta_1 + \beta \| \theta - \theta' \| \right) \notag\\
    &\leq \alpha_1^2 (\lambda_1(\theta) +\beta \|\theta - \theta'\|) + (\sum_{i=2}^d \alpha_i^2) (\lambda_1(\theta) - \Delta_1 + \beta \| \theta - \theta' \|) \notag \\
    &\leq \lambda_1(\theta) + \beta\| \theta - \theta'\| - \Delta_1 (1-\alpha_1^2) \label{equation::upper_bound_1}
\end{align}
And therefore, combining Equation \ref{equation::lower_bound_1} and \ref{equation::upper_bound_1}:
\begin{equation*}
    \lambda_1(\theta) - \beta\| \theta - \theta'\| \leq e_1(\theta)^\top H' e_1(\theta) \leq \lambda_1(\theta) + \beta\| \theta - \theta'\| - \Delta_1 (1-\alpha_1^2)
\end{equation*}
Therefore:
\begin{equation*}
    \alpha_1^2 \geq \frac{\Delta_1 - 2\beta\|\theta - \theta'\|  }{\Delta_1} = 1 - \frac{2\beta \| \theta - \theta'\|}{\Delta_1}
\end{equation*}
This in turn implies that $\sum_{i=2}^d \alpha_i^2 \leq \frac{2\beta \| \theta - \theta'\|}{\Delta_1}$ and that $1-\alpha \leq 1-\alpha^2 \leq \frac{2\beta \| \theta - \theta'\|}{\Delta_1} $. Therefore:
\begin{align*}
\| e_1(\theta) - e_1(\theta') \|^2  &= (1-\alpha_1)^2 + \sum_{i=2}^d \alpha_2^2  \\
&\leq (1-\alpha_1)^2 + \frac{2\beta \| \theta - \theta'\|}{\Delta_1} \\
&\leq \frac{ 4\beta^2 \| \theta - \theta'\|^2 }{\Delta^2_1} + \frac{2\beta \| \theta - \theta'\|}{\Delta_1}\\
&\leq \frac{ 4\beta \| \theta - \theta'\|}{\Delta_1}
 \end{align*}
The result follows.

\end{proof}

As a direct implication of Lemma \ref{lemma::lipschitzness_eigenvector}, we conclude the eigenvector function is continuous.


\subsection{Convergence rates for finding a new eigenvector, eigenvalue pair}\label{section::convergence_eigenvalue_eigenvector}

Let $\theta' = \theta +\Delta_\theta$, ridge riding minimizes the following loss w.r.t $e$ and $\lambda$ to find a candidate $e'$ and $\lambda'$:

\begin{equation}
    L(e, \lambda;\theta') = \| (1/\lambda) \mathcal{H}(\theta') e / \| e\| - e/\|e\|\|^2
\end{equation}

Notice that:
\begin{align*}
 L(e, \lambda; \theta') &= \frac{1}{ \lambda^2 \| e \|^2} e^\top \mathcal{H}(\theta')^2 e + 1 - 2\frac{1}{\lambda \| e \|^2 } e^\top \mathcal{H}(\theta') e 
\end{align*}

Therefore:
\begin{align*}
    \nabla_e L(e, \lambda; \theta') &=  \frac{1}{\lambda^2}\nabla_{e} \left(   \frac{1}{ \| e \|^2} e^\top \mathcal{H}(\theta')^2 e   \right) -2\frac{1}{\lambda}\nabla_e \left(\frac{1}{\| e \|^2} e^\top \mathcal{H}(\theta') e  \right)\\
    &=  \frac{2}{\lambda^2\| e\|}\left(  \mathcal{H}^2(\theta') - \tilde{e}^\top \mathcal{H}^2(\theta') \tilde{e} I \right) \tilde{e} - \frac{4}{\lambda \| e \|} \left( \mathcal{H}(\theta') - \tilde{e}^\top \mathcal{H}(\theta') \tilde{e} I  \right) \tilde{e} \\
    &= \left( \frac{2}{\lambda^2 \| e \|} \mathcal{H}^2(\theta') - \frac{4}{\lambda \| e \| } \mathcal{H}(\theta') \right)\tilde{e} + \left(\frac{4}{\lambda \| e \| } \tilde{e}^\top \mathcal{H}(\theta') \tilde{e}I - \frac{2}{\lambda^2 \| e\|} \tilde{e}^\top \mathcal{H}^2(\theta') \tilde{e}I     \right)\tilde{e} 
\end{align*}

Where $\tilde{e} = \frac{e}{\| e \|}$. 

Now let's compute the following gradient:
\begin{equation*}
\nabla_\lambda L(e, \lambda; \theta') = - \frac{2}{\lambda^{3} \| e\|^2 } e^\top \mathcal{H}(\theta')^2 e + \frac{2}{\lambda^2 \|e\|^2 } e^\top \mathcal{H}(\theta') e = \frac{2}{\lambda^2 \| e\|} \left(  e^\top \mathcal{H}(\theta') e - \frac{e^\top \mathcal{H}(\theta')^2 e}{\lambda }     \right)
\end{equation*}

We consider the following algorithm:

\begin{enumerate}
    \item Start at $(e, \lambda)$.
    \item Take a gradient step $e \rightarrow e - \alpha_e \nabla_e L(e, \lambda; \theta')$.
    \item Take a gradient step $\lambda \rightarrow  \lambda - \alpha_\lambda \nabla_\lambda L(e, \lambda; \theta')$. 
    \item Normalize $e \rightarrow \frac{e}{\| e\|}$.
\end{enumerate}

It is easy to see that the update for $e$ takes the form:

\begin{align*}
    e &\rightarrow \underbrace{\left( \left(  1 + \alpha_e\left(  \frac{2e^\top \mathcal{H}(\theta')^2 e }{ \lambda^2}  -  \frac{4e^\top \mathcal{H}(\theta') e}{\lambda}\right)    \right) I+ \alpha_e\left( \frac{ 4\mathcal{H}(\theta') }{\lambda}  - \frac{2\mathcal{H}(\theta')^2 }{\lambda^2}\right) \right)}_{U}e \\
    e&\rightarrow \frac{e}{\| e\|}
\end{align*}

Where we think of $U$ as an operator acting on the vector $e$. In fact if we consider $T$ consecutive steps of this algorithm, yielding normalized eigenvector candidates $e_0, \cdots, e_T$ and eigenvalue candidates $\lambda_0, \cdots, \lambda_T$, and name the corresponding $U-$operators as $U_1, \cdots, U_T$ it is easy to see that:
\begin{equation*}
    e_T = \frac{E_T}{\| E_T \|}
\end{equation*}
Where $E_T = \left(\prod_{i=1}^T U_T \right)e_0$. In other words, the normalization steps can be obviated as long as we normalize at the very end. This observation will prove useful in the analysis.

Let's assume $L$ is $\beta-$smooth and let's say we are trying to find the $i-$th eigenvalue eigenvector pair for $\theta'$: $(e_i(\theta'), \lambda_i(\theta'))$. Furthermore let's assume we start our optimizaation at the $(e_i(\theta), \lambda_i(\theta))$ pair. Furthermore, assume that $\theta'$ is such that:

\begin{equation}\label{equation::assumption_on_eigenval_eigenvec}
    \| e_i(\theta) - e_i(\theta')\| \leq \min(\frac{\Delta_i}{4}, \frac{\Delta_{i-1}}{4})  \text{  and  }  \| \lambda_i(\theta) - \lambda_i(\theta')\| \leq \min(\frac{\Delta_i}{4}, \frac{\Delta_{i-1}}{4})  
\end{equation}
Where $\lambda_{i}(\theta) - \lambda_{i}(\theta) = \Delta_{i-1} > 0$ and $\lambda_{i}(\theta) - \lambda_{i+1}(\theta) = \Delta_{i} > 0$. The existence of such $\theta'$ as in \ref{equation::assumption_on_eigenval_eigenvec} can be guaranteed by virtue of Lemmas \ref{proposition::eigenvalue_lipschitz} and \ref{lemma::lipschitzness_eigenvector}.


     

Notice that as long as $\alpha_e = \min(1/4, \Delta_{i}, \Delta_{i-1})$ is small enough the operator $U$ attains the form:
\begin{equation*}
    U = AI + \alpha_e \left( \frac{ 4\mathcal{H}(\theta') }{\lambda}  - \frac{2\mathcal{H}(\theta')^2 }{\lambda^2}\right) 
\end{equation*}

Where $\alpha_e / A $ is small.  

Notice that the operator $\left( \frac{ 4\mathcal{H}(\theta') }{\lambda}  - \frac{2\mathcal{H}(\theta')^2 }{\lambda^2}\right) $ has the following properties:

\begin{enumerate}
    \item $\left( \frac{ 4\mathcal{H}(\theta') }{\lambda}  - \frac{2\mathcal{H}(\theta')^2 }{\lambda^2}\right) $ has the exact same eigenvectors set $\{e_j(\theta') \}_{j=1}^T$ as $\mathcal{H}(\theta')$. 
    \item The eigenvalues of $\left( \frac{ 4\mathcal{H}(\theta') }{\lambda}  - \frac{2\mathcal{H}(\theta')^2 }{\lambda^2}\right) $ equal $\{\frac{4\lambda_j(\theta')}{\lambda} - \frac{2\lambda_j(\theta'^2}{\lambda^2}\}_{j=1}^d$.
\end{enumerate}

Consequently, if $|\lambda - \lambda_i(\theta')| < \min(\frac{\Delta'_i}{4}, \frac{\Delta'_{i-1}}{4}) \ )$ we conclude that the maximum eigenvalue of $\left( \frac{ 4\mathcal{H}(\theta') }{\lambda}  - \frac{2\mathcal{H}(\theta')^2 }{\lambda^2}\right) $ equals $\frac{4\lambda_i(\theta')}{\lambda} - \frac{2\lambda_i(\theta')^2}{\lambda^2}$ with eigenvector $e_i(\theta')$. 

Furthermore, the eigen-gap between the maximum eigenvalue and any other one is lower bounded by $\frac{\min(\Delta_i, \Delta_{i-1})}{2}$. Therefore, after taking a gradient step on $e$, the dot product $\langle e_i(\theta'), e_t \rangle = \gamma_t$ satsifies $\gamma^2_{t+1} \rightarrow  \gamma^2_t + \gamma^2_t *(1-\alpha^2_e\frac{\min(\Delta_i, \Delta_{i-1}^2)}{4} )$ 

If $|\lambda_t - \lambda_i(\theta')  |< \min(\frac{\Delta'_i}{4}, \frac{\Delta'_{i-1}}{4}) \ ) $, and the eigenvalue update satisfied the properties above, then $\lambda_{t+1}$ is closer to $\lambda_i(\theta')$ than $\lambda_t$, thus maintaining the invariance. We conclude that the convergence rate is the rate at which $\gamma_t \rightarrow 1$, which is captured by the following theorem:
\begin{theorem}
If $L$ is $\beta-$smooth, $\alpha_e = \min(1/4, \Delta_{i}, \Delta_{i-1})$, and $\| \theta - \theta'\| \leq \frac{\min(1/4, \Delta_{i}, \Delta_{i-1})}{\beta} $ then $|\langle e_t, e_i(\theta') \rangle| \geq 1- \left(1-\frac{\min(1/4, \Delta_{i}, \Delta_{i-1})}{4} \right)^{t} $
\end{theorem}

\subsection{Staying on the ridge}\label{subsection::staying_on_ridge}

In this section, we show that under the right assumptions on the step sizes, Ridge Riding stays along a descent direction. 

We analyze the following setup. Starting at $\theta$, we move along negative eigenvector $e_i(\theta)$ to $\theta' = \theta - \alpha e_i(\theta)$. Once there we move to $\theta'' = \theta' - \alpha e_i(\theta')$. Let $L : \Theta \rightarrow \mathbb{R}$ be the function we are trying to optimize. We show that:

\begin{theorem}
Let $L: \Theta \rightarrow \mathbb{R}$ have $\beta-$smooth Hessian, let $\alpha$ be the step size. If at $\theta$ RR satisfies:
$\langle \nabla L(\theta), e_i(\theta) \rangle \geq  \| \nabla L(\theta) \| \gamma $, and $\alpha \leq \frac{\min(\Delta_i, \Delta_{i-1}) \gamma^2 }{16\beta} $ then after two steps of RR:
\begin{equation*}
L(\theta'') \leq     L(\theta)  - \gamma \alpha\| \nabla L(\theta) \| 
\end{equation*}
\end{theorem}

\begin{proof}

Since $L$ is $\beta-$smooth, the third order derivatives of $L$ are uniformly bounded. Let's write $L(\theta')$ using a Taylor expansion:
\begin{align*}
    L(\theta') &= L(\theta - \alpha e_i(\theta))\\
    &\stackrel{(i)}{\leq} L(\theta) + \langle \nabla L(\theta), -\alpha e_i(\theta) \rangle + \frac{1}{2}(-\alpha e_i(\theta)^\top \mathcal{H}(\theta) (-\alpha e_i(\theta)) + c' \alpha^3  \beta \\
   &= L(\theta) - \alpha \langle \nabla L(\theta), e_i(\theta) \rangle+ \alpha^2 \frac{\lambda_i(\theta)}{2} + c' \alpha^3 \beta
\end{align*}
Inequality $(i)$ follows by Hessian smoothness. 

Let's expand $L(\theta'')$:
\begin{align*}
    L(\theta'') &= L(\theta' - \alpha e_i(\theta') ) \\
        &\leq L(\theta')  - \alpha \langle \nabla L(\theta'), e_i(\theta') \rangle + \frac{\alpha^2}{2} \lambda_i(\theta') + c''\alpha^3 \beta \\
        &\leq L(\theta)  - \alpha \langle  \nabla L(\theta), e_i(\theta) \rangle- \alpha \langle \nabla L(\theta'), e_i(\theta') \rangle + \frac{\alpha^2\lambda_i(\theta)}{2}  +\frac{\alpha^2 \lambda_i(\theta')}{2}  + (c'+c'')\alpha^3 \beta
\end{align*}
Notice that for any $v\in \mathbb{R}^d$ it follows that $\langle \nabla L(\theta') , v\rangle  \leq  \langle \nabla L(\theta), v\rangle - \alpha v^\top  \nabla^2 L(\theta) e_i(\theta)  + c''' \beta \alpha^2 = \langle \nabla L(\theta), v\rangle - \alpha \lambda_i(\theta) v^\top  e_i(\theta)  + c''' \beta \alpha^2$. Plugging this in the sequence of inequalities above:
\begin{align}
    L(\theta'') &\leq L(\theta) - \alpha \langle \nabla L(\theta) , e_i(\theta) \rangle - \alpha \langle \nabla L(\theta'), e_i(\theta') \rangle + \frac{\alpha^2\lambda_i(\theta)}{2}  +\frac{\alpha^2 \lambda_i(\theta')}{2}  + (c'+c'')\alpha^3 \beta \notag\\
    &\leq L(\theta)  - \alpha \langle \nabla L(\theta), e_i(\theta) \rangle- \alpha \langle \nabla L(\theta), e_i(\theta') \rangle +\frac{3\alpha^2 \lambda_i(\theta) + \alpha^2 \lambda_i(\theta')}{2} + (c' + c'' + c''')\alpha^3 \beta\notag \\
    &= L(\theta)  - 2\alpha \langle \nabla L(\theta), e_i(\theta) \rangle+ \alpha \langle \nabla L(\theta), e_i(\theta) - e_i(\theta') \rangle  +\frac{3\alpha^2 \lambda_i(\theta) + \alpha^2 \lambda_i(\theta')}{2} + (c' + c'' + c''')\alpha^3 \beta\notag\\
    &\leq L(\theta)  - 2\alpha \langle \nabla L(\theta), e_i(\theta) \rangle + \alpha \langle \nabla L(\theta), e_i(\theta) - e_i(\theta') \rangle  +\frac{3\alpha^2 \lambda_i(\theta) + \alpha^2 \lambda_i(\theta')}{2} + (c' + c'' + c''')\alpha^3 \beta\notag\\
    &\stackrel{(i)}{\leq} L(\theta)  - 2\alpha \langle \nabla L(\theta), e_i(\theta) \rangle + \alpha \| \nabla L(\theta\| \sqrt{\frac{4\beta\| \theta - \theta' \|}{\min(\Delta_i, \Delta_{i-1})}}  +\frac{3\alpha^2 \lambda_i(\theta) + \alpha^2 \lambda_i(\theta')}{2} + (c' + c'' + c''')\alpha^3 \beta \label{equation::upper_bound_Lprimeprime}
\end{align}
Where inequality $(i)$ follows from Cauchy-Schwarz and Lemma \ref{lemma::lipschitzness_eigenvector} since:

$ \| e_i( \theta) - e_i(\theta') \|\leq \sqrt{\frac{ 4\beta \| \theta - \theta'\|}{\min(\Delta_i, \Delta_{i-1})}}  = \sqrt{\frac{ 4\beta \alpha}{\min(\Delta_i, \Delta_{i-1})}} $

Recall that by assumption $\langle \nabla L(\theta), e_i(\theta) \rangle \geq  \| \nabla L(\theta) \| \gamma $ and $\gamma \in (0,1)$ and that $\alpha \leq \frac{\min(\Delta_i, \Delta_{i-1}) \gamma^2 }{16\beta} $. Applying this to inequality \ref{equation::upper_bound_Lprimeprime} :

\begin{align*}
    L(\theta'') &\leq L(\theta)  - 2\alpha \langle \nabla L(\theta), e_i(\theta) \rangle + \alpha \| \nabla L(\theta\| \sqrt{\frac{4\beta\alpha}{\min(\Delta_i, \Delta_{i-1})}}  +\frac{3\alpha^2 \lambda_i(\theta) + \alpha^2 \lambda_i(\theta')}{2} + (c' + c'' + c''')\alpha^3 \beta \\
    &\leq L(\theta)  - 2\alpha \langle \nabla L(\theta), e_i(\theta) \rangle + \frac{ \gamma \alpha}{2} \| \nabla L(\theta\| +\frac{3\alpha^2 \lambda_i(\theta) + \alpha^2 \lambda_i(\theta')}{2} + (c' + c'' + c''')\alpha^3 \beta \\
    &\leq  L(\theta)  - \gamma \alpha \| \nabla L(\theta) \| +\frac{3\alpha^2 \lambda_i(\theta) + \alpha^2 \lambda_i(\theta')}{2} + (c' + c'' + c''')\alpha^3 \beta\\
    &\leq  L(\theta)  - \gamma \alpha\| \nabla L(\theta) \| 
\end{align*}

The last inequality follows because term  $\frac{3\alpha^2 \lambda_i(\theta) + \alpha^2 \lambda_i(\theta')}{2}  \leq 0$ and of order less than the third degree terms at the end.

\end{proof}

\subsection{Behavior of RR near a saddle point}








The discussion in this section is intended to be informal and has deliberately been written in this way. First, let $\theta_0$ be a saddle point of $\mathcal{L}(\theta)$, and consider the steps of RR $\theta_1, ..., \theta_t, \cdots$. Let $H$ be the Hessian of $\mathcal{L}$ at $\theta_0$. We will start by using the first-order Taylor expansion, $\nabla_\theta \mathcal{L}(\theta_t) = \mathcal{H}(\theta_{t-1}) (\theta_t - \theta_{{t-1}}) + o(\epsilon^2)$ ignoring the error term to approximate the gradient close to $\theta_{t-1}$.

We will see that $\nabla_\theta \mathcal{L}(\theta_t) =  \alpha \sum_{l=0}^{t-1} \lambda_i(\theta_l) e_i(\theta_l) + o(t\epsilon^2)$ for all $t$. We proceed by induction. Notice that for $t = 1$, this is true since $\nabla_\theta \mathcal{L}(\theta_1) =\mathcal{ H}(\theta_0)(\theta_1-\theta_0) + o(\epsilon^2) = \alpha \lambda_i(\theta_0)) e_i(\theta_0) + o(\epsilon^2)  $ for some lower order error term $\epsilon$. 

Now suppose that for some $t\geq 1$ we have $\nabla_\theta \mathcal{L}(\theta_t) =  \alpha \sum_{l=0}^{t-1} \lambda_i(\theta_l) e_i(\theta_l) + o(t\epsilon^2)$, this holds for $t=0$. By a simple Taylor expansion around $\theta_t$:
\begin{align*}
    \nabla_\theta \mathcal{L}(\theta_{t+1}) &= \nabla_\theta \mathcal{L}(\theta_{t}) +  \mathcal{H}( \theta_t ) (\theta_{t+1} - \theta_t) + o(\epsilon^2) \\
    &\stackrel{(i)}{=} \alpha \sum_{l=0}^{t-1} \lambda_i(\theta_l) e_i(\theta_l) + o(t\epsilon^2)  +  \mathcal{H}( \theta_t ) (\theta_{t+1} - \theta_t) + o(\epsilon^2)\\
    &= \alpha \sum_{l=0}^{t} \lambda_i(\theta_l) e_i(\theta_l) + o((t+1)\epsilon^2)  
\end{align*}
Equality $(i)$ follows from the inductive assumption. The last inequality follows because by definition $\mathcal{H}( \theta_t ) (\theta_{t+1} - \theta_t) = \alpha \lambda_i(\theta_t) e_i(\theta_t)$. The result follows.  

\subsection{Symmetries lead to repeated eigenvalues}\label{subsection::repeated_eigenvalues}

Let $\mathcal{L} : \R^d \to \R^d$ be a twice-differentiable loss function and write $[n] = \{1, \ldots, n\}$ for $n \in \mathbb{N}$. For any permutation $\phi \in S_d$ (the symmetric group on $d$ elements), consider the group action
$$ \phi(\th_1, \ldots, \th_d) = \left(\th_{\phi(1)}, \ldots, \th_{\phi(d)}\right) $$
and abuse notation by also writing $\phi : \R^d \to \R^d$ for the corresponding linear map. For any $N, m \in \mathbb{N}$, define the group of permutations
$$ \Phi_N^m = \left\{ \prod_{i=1}^{m} (i,i+mk) \mid k\in [N-1] \right\}\footnote{For $m=1$ we have $\Phi_N^1 = \{ (1,2), \ldots, (1,N) \}$, which together generate all $N!$ permutations on $N$ elements. For larger $m$, $\Phi_N^m$ also generates $N!$ permutations, but the $m$ elements within each set are tied to each other. For instance,
$ \Phi_3^2 = \{ (1, 3)(2, 4), (1, 5)(2,6) \} $
which together generate
$ \{ (1), (1, 3)(2, 4), (1, 5)(2,6), (3,5)(4,6), (1, 3, 5)(2, 4, 6), (1,5,3)(2,6,4) \} \,. $} \,. $$
Now assume there are $N$ non-overlapping sets
$$ \{\theta_{k_i^1}, \ldots, \theta_{k_i^m} \} $$
of $m$ parameters each, with $i \in [N]$, which we can reindex (by reordering parameters) to
$$  \{\theta_{1+m(i-1)}, \ldots, \theta_{mi}\} $$
for convenience. Assume the loss function is invariant under all permutations of these $N$ sets, namely, $\mathcal{L} \circ \phi = \mathcal{L}$ for all $\phi \in \Phi_N^m$. Our main result is that such parameter symmetries reduce the number of distinct Hessian eigenvalues, cutting down the number of directions to explore by ridge riding.

\begin{theorem}
Assume that for some $N,m$ we have $\mathcal{L}\circ\phi = \mathcal{L}$ and $\phi(\th) = \th$ for all $\phi \in \Phi_N^m$. Then $\nabla^2 \mathcal{L}(\th)$ has at most $d-m(N-2)$ distinct eigenvalues.
\end{theorem}

We first simplify notation and prove a few lemmata.

\begin{definition}
Write $\Phi = \Phi_N^m$ and $H = \nabla^2 \mathcal{L}(\th)$. We define an eigenvector $v$ of $H$ to be \emph{trivial} if $\phi(v) = v$ for all $\phi \in \Phi$. We call an eigenvalue \emph{trivial} if all corresponding eigenvectors are trivial.
\end{definition}

\begin{lemma}
Assume $\mathcal{L} \circ \phi = \mathcal{L}$ for some $\phi \in \Phi$ and $\phi(\th) = \th$. If $(v, \la)$ is an eigenpair of $H$ then so is $(\phi(v), \la)$.
\end{lemma}

\begin{proof}
First notice that $D\phi$ (also written $\nabla \phi$) is constant by linearity of $\phi$, and orthogonal since
\begin{align*}
(D\phi^T D\phi)_{ij} = \sum_k D\phi_{ki}D\phi_{kj} = \sum_k \delta_{\phi(k)i}\delta_{\phi(k)j} = \delta_{ij}\sum_k\delta_{\phi(k)i} = \delta_{ij} = I_{ij} \,.
\end{align*}
Now applying the chain rule to $\mathcal{L} = \mathcal{L}\circ \phi$ we have
$$ D\mathcal{L} = D\mathcal{L}|_\phi \circ D\phi $$
and applying the product rule and chain rule again,
$$ D^2 \mathcal{L} = D(D\mathcal{L}|_\phi \circ D\phi) = D\phi^T D^2 \mathcal{L}|_\phi D\phi + 0 $$
since $D^2 \phi = 0$. If $\phi(\th) = \th$ then we obtain
$$ H = D\phi^T H D\phi \,, \quad \text{or equivalently,} \quad H = D\phi H D\phi^T $$
by orthogonality of $D\phi$. Now notice that $\phi$ acts linearly as a matrix-vector product
$$ \phi(v) = D\phi \cdot v \,, $$
so any eigenpair $(v, \la)$ of $H$ must induce
$$ H \phi(v) = (D\phi H D\phi^T)(D\phi v) = D\phi H v = D\phi \la v = \la D\phi v = \la \phi(v) $$
as required.
\end{proof}

\begin{lemma}
Assume $v$ is a non-trivial eigenvector of $H$ with eigenvalue $\la$. Then $\la$ has multiplicity at least $N-1$.
\end{lemma}

\begin{proof}
Since $v$ is non-trivial, there exists $\phi \in \Phi$ such that $\phi(v)_i \neq v_i$ for some $i \in [d]$. Without loss of generality, by reordering the parameters, assume $i = 1$. Since $\phi = \prod_{i=1}^{m} (i,i+mk)$ for some $k \in [N-1]$, we can set $k$ to $N-1$ after reindexing of the $N$ sets. Now $u = v-\phi(v)$ is an eigenvector of $H$ with $u_1 \neq 0$ and zeros everywhere except the first and last $m$ entries, since $k=N-1$ implies that $\phi$ keeps other entries fixed. We claim that
$$ \left\{ \prod_{i=1}^{m} (i,i+mk) u \right\}_{k=0}^{N-2} $$
are $N-1$ linearly independent vectors. Assume there are real numbers $a_0, \ldots, a_{N-2}$ such that
$$\sum_{k=0}^{N-2} a_k \prod_{i=1}^{m} (i,i+mk) u = 0 \,. $$
In particular, noticing that $u_{1+mj} = 0$ for all $1 \leq j \leq N-2$ and considering the $(1+mj)$th entry for each such $j$ yields
$$0 = \sum_{k=0}^{N-2} a_k \left(\prod_{i=1}^{m} (i,i+mk) u\right)_{1+mj} = \sum_{k=0}^{N-2} a_k \del_{jk} u_1 = a_j u_1 \,. $$
This implies $a_j=1$ for all $1 \leq j \leq N-2$. Finally we are left with
$$ 0 = a_0 \prod_{i=1}^{m} (i,i) u = a_0u $$
which implies $a_0 = 0$, so the vectors are linearly independent. By the previous lemma, each vector is an eigenvector with eigenvalue $\la$, so the eigenspace has dimension at least $N-1$ as required.
\end{proof}

\begin{lemma}
There are at most $d-m(N-1)$ linearly independent trivial eigenvectors.
\end{lemma}

\begin{proof}
Assume $v$ is a trivial eigenvector, namely, $\phi(v) = v$ for all $\phi \in \Phi$. Then $v_i = v_{i+mk}$ for all $1 \leq i \leq m$ and $1 \leq k \leq N-1$, so $v$ is fully determined by its first $m$ entries $v_1, \ldots, v_m$ and its last $d-mN$ entries $v_{mN+1}, \ldots, v_d$. This implies that trivial eigenvectors have at most $m + d-mN = d-m(N-1)$ degrees of freedom, so there can be at most $d-m(N-1)$ linearly independent such vectors.
\end{proof}

The theorem now follows easily.

\begin{proof}
Let $k$ and $l$ respectively be the number of distinct trivial and non-trivial eigenvalues. Eigenvectors with distinct eigenvalues are linearly independent, so $k \leq d-m(N-1)$ by the previous lemma. Now assuming for contradiction that $k+l > d-m(N-2)$ implies
$$ d-m(N-2) < k+l \leq d-m(N-1)+l \quad \implies \quad l > m \,.$$
On the other hand, each non-trivial eigenvalue has multiplicity at least $N-1$, giving $k+l(N-1) \leq d$ linearly independent eigenvectors. We obtain the contradiction
$$ d \geq k+l(N-1) = k+l+l(N-2) > d-m(N-2)+l(N-2) > d $$
and conclude that $k+l \leq d-m(N-2)$, as required.
\end{proof}

\subsection{ Maximally Invariant Saddle }
\label{subsection::MIS}

In this section we show that for the case of tabular RL problems, the Maximally Invariant Saddle (MIS) corresponds to the parameter achieving the optimal reward and having the largest entropy. 

We consider $\theta \in \mathbb{R}^{|S|\times |A|}$ the parametrization of a policy $\pi_\theta$ over an MDP with states $S$ and actions $A$. We assume $\theta = \{ \theta_{s} \}_{s \in S}$ with  $\theta_s \in \mathbb{R}^{|A|}$ and (for simplicity) satisfying\footnote{A similar argument follows for a softmax parametrization.} $\sum_{a \in A} \theta_{s,a} = 1$ and $\theta_{s,a} \geq 0$.

Let $\Phi$ denote the set of symmetries over parameter space. In other words, $\phi \in \Phi$ if $\phi$ is a permutation over $|S|\times |A|$ and for all $\theta$ a valid policy parametrization, we have that $J(\theta) = J(\phi(\theta))$ such that $\phi(\theta) = \{ \phi(\theta_s)\}_{s \in S}$ acting per state.

We also assume the MDP is episodic in its state space, meaning the MDP has a horizon length of $H$ and each state $s \in S$ is indexed by a horizon position $h$. No state is visited twice during an episode.

We show the following theorem:

\begin{theorem}\label{theorem::value_invariant_saddle}
Let $\Theta_b$ be the set of parameters that induce policies satisfying $J(\theta) = b$ for all $\theta \in \Theta_b$. Let $\theta^* \in \Theta_b$ be the parameter satisfying $\theta^* = \arg\max_{\theta \in \Theta_b} \sum_{s} H(\pi_\theta (\mathbf{a}| s) )$.  Then for all $\phi \in \Phi$ it follows that $\phi(\theta^*) = \theta_*$.
\end{theorem}

\begin{proof}
 Let $\theta \in \Theta_b$ and let's assume there is a $\theta' \in \Theta_b$ such that $\phi(\theta') \neq \theta$. We will show there must exist $\theta'' \in \Theta_b$ such that $\sum_{s} H(\pi_{\theta''} (\mathbf{a}| s) ) > \max\left( \sum_{s} H(\pi_\theta (\mathbf{a}| s) ) , \sum_{s} H(\pi_{\theta'} (\mathbf{a}| s) )\right) $. 

Let $s$ be a state such that $\theta_s \neq \theta'_s$ and having maximal horizon position index $h$. In this case, all states $s'$ with a horizon index larger than $h$ satisfy $\theta_{s'} = \phi(\theta_{s'})$. Therefore for any $s'$ having index $h+1$ (if any) it follows that the value function $V_\theta(s') = V_{\theta'}(s')$. Since the symmetries hold over any policy and specifically for delta policies, it must be the case that at state $s$ and for any $a, a' \in A$ such that there is a $\phi' \in \Phi$ with (abusing notation) $\phi'(s,a) \rightarrow s,a'$ it must hold that $Q_\theta(s,a) = Q_\theta(s,a')$. Therefore the whole orbit of $a$ under $\phi'$ for any $\phi' \in A$ has the same $Q$ value under $\theta$ and $\theta'$. Since the entropy is maximized when all the probabilities of these actions are the same, this implies that if $\theta$ does not correspond to a policy acting uniformly over the orbit of $a$ at state $s$ we can increase its entropy by turning it into a policy that acts uniformly over it. Applying this argument recursively down the different layers of the episodic MDP implies that for any $a \in A$, the maximum entropy $\theta \in \Theta_b$ assigns a uniform probability over all the actions on $a'$s orbit. It is now easy to see that such a policy must satisfy $\phi(\theta) = \theta$ for all $\phi \in \Phi$.



\end{proof}

We now show a result relating the entropy regularized gradient-norm objective:

\begin{align*}
    \arg\min_{\theta}  |\nabla_{\theta} J(\theta)|  - \lambda H(\pi_\theta(\mathbf{a})), \lambda > 0
\end{align*}

In this discussion we will consider a softmax parametrization for the policies. Let's start with the following lemma:

\begin{lemma}
Let $p_i(\theta) = \frac{\exp(\theta_i)}{\sum_j \exp(\theta_j))} $ parametrize a policy over $K$ reward values $\{r_i\}_{i=1}^K$. The value function's gradient satisfies:

\begin{equation*}
\left( \nabla \sum_{j =1}^K p_j(\theta) \right)_i = p_\theta(i)(r_i - \bar{r})
\end{equation*}
Where $\bar{r}= \sum_{j=1}^K  p_i(\theta) r_i$.
\end{lemma}
\begin{proof}
Let $Z(\theta)  = \sum_{j=1}^K \exp(\theta_j)$. The following equalities hold:
\begin{align*}
    \left( \nabla \sum_{j =1}^K p_j(\theta) \right)_i &= \frac{Z(\theta) \exp(\theta_i)r_i - \exp^2(\theta_i) r_i}{Z^2(\theta)} + \sum_{j\neq i} \frac{-\exp(\theta_j) r_j \exp(\theta_i)}{Z^2(\theta)} \\
    &= \frac{Z(\theta)\exp(\theta_i)r_i}{Z^2(\theta)} - \sum_j \exp(\theta_i) \frac{\exp(\theta_j)}{Z^2(\theta)} \\
    &= \frac{\exp(\theta_i)r_i}{Z(\theta)} - \frac{\exp(\theta_i)}{Z(\theta)} \left(\sum_{j} \frac{\exp(\theta_j) r_j}{Z(\theta)} \right) \\
    &= p_i(\theta)\left(r_i - \bar{r}\right).
\end{align*}

The result follows.
\end{proof}

We again consider an episodic MDP with horizon length of $H$ and such that each state $s \in S$ is indexed by a horizon position $h$. No state is visited twice during an episode. Recall the set of symmetries is defined as $\phi \in \Phi$ if for any policy $\pi: \mathcal{S} \rightarrow \Delta_{A}$, with $Q-$function $Q_\pi : \mathcal{S} \times \mathcal{A} \rightarrow \mathbb{R}$, it follows that:

\begin{equation}
    Q_\pi( \phi(s),\phi(a)) = Q_\pi(s,a). 
\end{equation}

We abuse notation and for any policy parameter $\theta \in  \mathbb{R}^{|\mathcal{S}| \times |\mathcal{A}|}$ we denote the parameter vector resulting of the action of a permutation on $\phi$ on the indices of a parameter vector $\theta$ by $\phi(\theta)$.


We show the following theorem:

\begin{theorem}
Let $\Theta_b$ be the set of parameters that induce policies satisfying $\|\nabla J(\theta)\| = b$ for all $\theta \in \Theta_b$. Let $\theta^* \in \Theta_b$ be a parameter satisfying $\theta^* = \arg\max_{\theta \in \Theta_b} \sum_{s} H(\pi_\theta (\mathbf{a}| s) )$ (there could be multiple optima).  Then for all $\phi \in \Phi$ it follows that $\phi(\theta^*) = \theta_*$.
\end{theorem}

\begin{proof}
Let $\theta \in \Theta_b$ and let's assume there is a $\phi \in \Phi$ such that  $\theta' = \phi(\theta) \neq \theta$.

We will show there must exist $\theta'' \in \Theta_b$ such that $\sum_{s} H(\pi_{\theta''} (\mathbf{a}| s) ) > \max\left( \sum_{s} H(\pi_\theta (\mathbf{a}| s) ) , \sum_{s} H(\pi_{\theta'} (\mathbf{a}| s) )\right) $.

Since we are assuming a softmax parametrization and $\phi$ is a symmetry of the MDP, it must hold that for any two states $s$ and $s'$ with (abusing notation) $s' = \phi(s)$:
\begin{equation*}
    E_{a \sim \pi_\theta}[  Q_{\pi(\theta)}(s, a) ] =  E_{a \sim \pi_{\phi(\theta)}}[  Q_{\pi(\phi(\theta))}(s, a) ] 
\end{equation*}
We conclude that the gradient norm $\|\nabla J(\phi(\theta))\|$ must equal that of $\|\nabla J(\theta)\|$. This implies that if $\phi(s) \neq \phi(s')$ and wlog $H(\pi_\theta(\mathbf{a}| s) ) > H(\pi_\theta(\mathbf{a}|\phi(s)))$ , then we can achieve the same gradient norm but larger entropy by substituting $\theta_{\phi(s)}$ with $\theta_s$. Where $\theta_s$ denotes the $|\mathcal{A}|$-dimensional vector of the policy parametrization for state $s$. The gradient norm would be preserved and the total entropy of the resulting policy would be larger of that achieved by $\theta$ and $\theta'$. This finalizes the proof.

\end{proof}

\end{document}